\documentclass[openacc]{rsproca_new}%%%%where rsproca is the template name
\usepackage{graphicx}
\usepackage{url}
\usepackage{amsbsy}
\usepackage{amssymb}
\usepackage{amsmath}
\usepackage{amsthm}
\usepackage{color}
\usepackage{tabularx}
\usepackage{booktabs}
\usepackage{sidecap}
\usepackage{lipsum}  
\usepackage{adjustbox}
\usepackage{floatrow}
\usepackage{caption}
\usepackage{subcaption}

\def\GP{{\mathcal {GP}}}
\def\cN{{\mathcal {N}}}

\def\dt{{\text dt}}

\def\t{{\mathbf t}}
\def\f{{\mathbf f}}
\def\x{{\mathbf x}}
\def\v{{\mathbf v}}

\def\eye{{\mathbf I}}
\def\dtau{{\text d\tau}}
\def\R{{\mathbb R}}

\def\y{{\mathbf y}}

\DeclareMathOperator\supp{supp}

\newcommand{\fourier}[1]{\mathcal{F} \left\{#1\right\}}
\newcommand{\invfourier}[1]{\mathcal{F}^{-1} \left\{#1\right\}}

\newcommand{\E}[1]{\mathbb E \left[#1\right]}
\newcommand{\MVN}[1]{\text{MVN} \left(#1\right)}
\newcommand{\V}[1]{\mathbb V \left[#1\right]}

\newtheorem{definition}{Definition}
\newtheorem{theorem}{Theorem}
\newtheorem{remark}{Remark}
\newtheorem{lemma}{Lemma}
\newtheorem*{lemma-non}{Lemma}

\newtheorem{example}{Example}

\newcommand{\red}[1]{{ #1}}
\newcommand{\blue}[1]{{\color{blue} #1}}

\def\titlehead{Research}%%Please insert respective article type here
\jname{rspa}
\Journal{Proc R Soc A\ }

\begin{document}

%%%% Article title to be placed here
\title{Gaussian Process Deconvolution}

\author{%%%% Author details
Felipe Tobar$^{1}$, Arnaud Robert$^{2}$ and Jorge F. Silva$^{3}$}

%%%%%%%%% Insert author address here
\address{
$^{1}$Initiative for Data \& Artificial Intelligence, Universidad de Chile\\
$^{2}$Department of Computing, Imperial College London\\
$^{3}$Department of Electrical and Electronic Engineering, Universidad de Chile\\
}

%%%% Subject entries to be placed here %%%%
\subject{machine learning, signal processing}

%%%% Keyword entries to be placed here %%%% 

%%%% Insert corresponding author and its email address}
\corres{Felipe Tobar\\
\email{ftobar@uchile.cl}}

%%%% Abstract text to be placed here %%%%%%%%%%%%
\begin{abstract}
%!TEX root = ../GPDC.tex
Let us consider the deconvolution problem, that is, to recover a latent source $x(\cdot)$ from the observations $\y = [y_1,\ldots,y_N]$ of a convolution process $y = x\star h + \eta$, where $\eta$ is an additive noise, \red{the observations in $\y$ might have missing parts with respect to $y$}, and the filter $h$ could be unknown. We propose a novel strategy to address this task when $x$ is a continuous-time signal: we adopt a Gaussian process (GP) prior on the source $x$, which allows for closed-form Bayesian nonparametric deconvolution. We first analyse the direct model to establish the conditions under which the model is well defined. Then, we turn to the inverse problem, where we study i) some necessary conditions under which Bayesian deconvolution is feasible, and ii) to which extent the filter $h$ can be learnt from data or approximated for the blind deconvolution case. The proposed approach, termed Gaussian process deconvolution (GPDC) is compared to other deconvolution methods conceptually, via illustrative examples, and using real-world datasets.

\end{abstract}
%%%%%%%%%%%%%%%%%%%%%%%%%%%

%%%%%%%%%% Insert the texts which can accomdate on firstpage in the tag "fmtext" %%%%%

\begin{fmtext}

\textbf{This is the author generated postprint of the accepted manuscript (i.e., the accepted version not typeset by the journal) produced to be shared in personal or public repositories.}

\end{fmtext}

%%%%%%%%%%%%%%% End of first page %%%%%%%%%%%%%%%%%%%%%

\maketitle

\section{Introduction} 
\label{sec:intro}

In signal processing, the convolution between a (continuous-time) source $x$ and a filter $h$, denoted by\footnote{For a lighter notation, we use the compact expressions $x,h$ and $f$ to represent the functions $(x(t))_{t\in\R}$, $(h(t))_{t\in\R}$, and $(f(t))_{t\in\R}$ respectively.}
\begin{equation}
	f(t)=x\star h = \int_\R x(\tau)h(t-\tau)\dtau, 	
\end{equation} 
can be understood as a generalised (noiseless) observation of $x$ through an acquisition device with impulse response $h$. 
Here, $h$ reflects the quality or precision of the observation device, since the "closer" $h$ is to a Dirac delta, the "closer" the convolved quantity $f$ is to the source $x$. 

\red{This convolution model, and the need to recover the source signal $x$ from a set of noisy observations of $f$, arise in a number of scenarios including: astronomy \cite{arxiv.2210.01666}, channel equalisation in telecommunications \cite{668635}, de-reverberation \cite{willardson2018time}, seismic wave reconstruction \cite{arya1978deconvolution}, and image restoration \cite{8197162} to name a few. }
In these scenarios, practitioners need to remove the unwanted artefacts introduced by the non-ideal filter $h$, in other words, they require to perform a \emph{deconvolution} to recover $x$ from $f$. \red{In practice, the deconvolution operates over observations that are noise-corrupted (due to sensing procedure) realisations of $f$, denoted $\y$ in our setting. Furthermore, we will \textit{de facto} assume that there are missing observations, since our formulation establishes the source $x$ and the convolution $f$ as continuous-time objects while the observations $\y$ are always finite meaning that there are "missing parts" in the observations.}  Fig.~\ref{fig:conv_diagram} illustrates this procedure for the image of a bird using the proposed method. 
\begin{figure}[h!]
\centering
\includegraphics[width=0.9\textwidth]{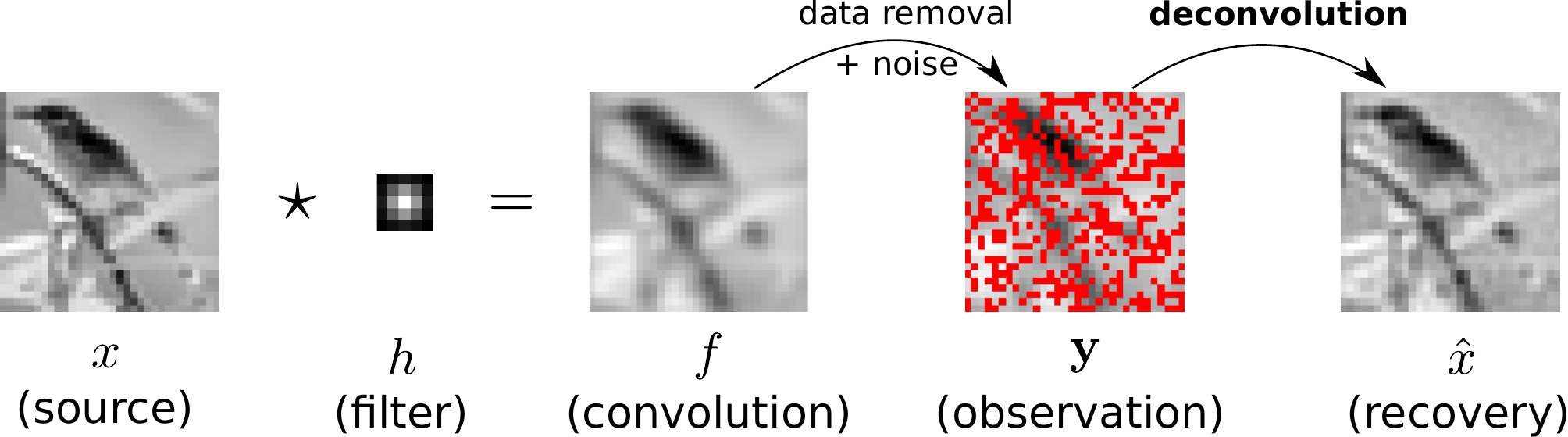}
\caption{Illustration of the deconvolution operation: an image $x$ is convolved with a filter $h$ to produce the blurry image $f$, which is then corrupted (additive noise and data removal in red) to yield $\y$. Deconvolution aims to produce an estimator of the original signal $x$ from $\y$ denoted $\hat{x}$. This image produced by the proposed method.}
\label{fig:conv_diagram}
\end{figure}

\red{As a particular instance of inverse problems, the deconvolution of corrupted signals has been largely addressed from a statistical perspective \cite{tarantola2005inverse,stuart2010inverse,aster2018parameter}}. This approach is intimately related to the linear filter theory, where the foundations laid by the likes of \cite{wiener1964extra,kalman1960new} are still at the core of modern-day implementations of deconvolution. A common finding across the vast deconvolution literature is that, under the presence of noisy or missing data, an adequate model for reconstructing $f$ (from $\y$) before the deconvolution is fundamental. From a Bayesian perspective, this boils down to an appropriate choice of the prior distribution of the source $x$; we proceed by imposing a Gaussian process prior on the source.

\subsection{Contribution and organisation} 

\red{Despite the ubiquity of the deconvolution problem and the attention it has received \cite{1457675}, we claim that the recovery of continuous-time signals from a finite number of corrupted (convolved) samples from a Bayesian standpoint, \textbf{which can provide error bars for the deconvolution}, has been largely underexplored.} With this challenge in mind, we study Bayesian nonparametric deconvolution using a \red{Gaussian process (GP) }prior over the latent source $x$, our method is thus termed Gaussian process deconvolution (GPDC).
The main contributions of our work include: i) the conditions for the proposed model to be well defined and how to generate samples from it; ii) the closed-form solution for the posterior deconvolution and when this deconvolution is possible; iii)  its application to the blind deconvolution and the required approximations; and  iv) experimental validation of our GPDC method on 1D and 2D real-world data.

The article is organised as follows. Sec.~\ref{sec:background} presents the convolution as a GP hierarchical model and the related literature. Sec.~\ref{sec:proposed_model} studies the \emph{direct model} (i.e., $x$ generates $f$ and $\y$) and defines the requirements of $h(t)$ and $x$ for $f(t)$ to be well defined point-wise. Then, Sec.~\ref{sec:deconvolution} focuses on the \emph{inverse problem} (i.e., $x$ is estimated from $\y$) and studies the recovery from the Fourier representation perspective. Sec.~\ref{sec:deconvolution-blind} addresses the blind deconvolution (i.e., $h$ is unknown), while Secs.~\ref{sec:exp} and \ref{sec:discussion} present the experimental validation and conclusions respectively.

%!TEX root = ../GPDC.tex

\section{Deconvolution using GPs} 

%\subsection{A GP Proposed model and relationship to existing literature} 
\label{sec:background}

\subsection{Proposed generative model}
Let us consider the following hierarchical model (see Fig.~\ref{fig:gm} for the graphical representation):
\begin{alignat}{3}
	&\text{source process:}  &&\quad x(t) &&\sim \GP(m_x(t), K_x(t)),\label{eq:source}\\
	&\text{convolution:}&&\quad f(t) &&= \int_\R x(\tau)h(\tau-t)\dtau, \label{eq:conv_proc}\\
	&\text{observations:}&&\quad\ \ \ y_i && \sim \cN(f(t_i),\sigma_n^2), i=1,\ldots,N, \label{eq:conv_y}
\end{alignat}
where $\t=\left\{t_i\right\}_{i=1}^N\in \mathbb{R}^N$ indicate the observation times.  
First, eq.~\eqref{eq:source} places a stationary GP prior on the source with covariance $K_x(t)$; we will assume $m_x(t)=0$ for all $t$. This selection follows the rationale that the prior distribution is key for implementing deconvolution under missing and noisy data, and the well-known interpolation properties of GPs \cite{Rasmussen:2006}. Second, eq.~\eqref{eq:conv_proc} defines the continuous-time convolution process $f$ through a linear and time-invariant filter $h$. Third, eq.~\eqref{eq:conv_y} defines a Gaussian likelihood, where the noisy observations of $f$ at times $\t$ are denoted by $\y = [y_1,\ldots, y_N] \in \mathbb{R}^N$. \red{Notice that the observations $\y$ only see "parts" of $f(\cdot)$ and thus a Bayesian approach is desired to quantify the uncertainty related to the conditional process $x|\y$.}

\begin{figure}
\centering
\includegraphics[width=0.6\textwidth]{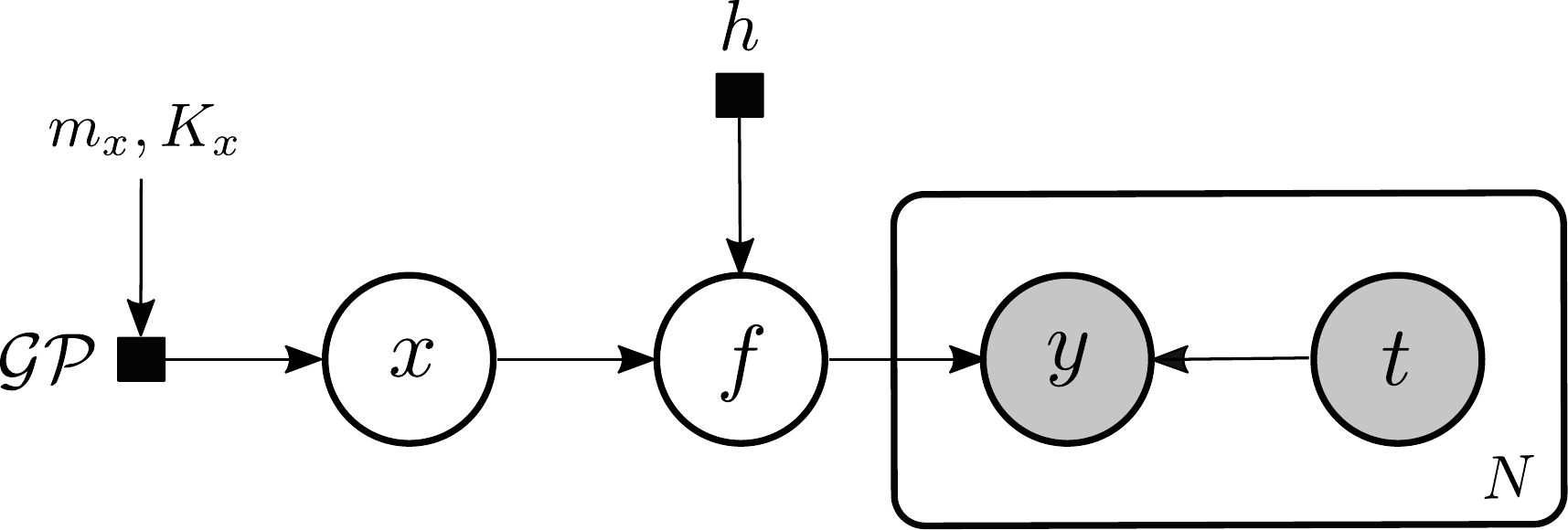}
\caption{Graphical model for the hierarchical GP convolution. Recall that observed and latent random variables are shown in grey and white circles respectively, while black squares denote fixed parameters.}
\label{fig:gm}
\end{figure}

\subsection{Relationship to classical methods and prior work}

%general bayesian deconv
Perhaps the simplest approaches to deconvolution are the \emph{Inverse FT} method, which performs deconvolution as a point-wise division between Fourier transforms $\invfourier{\fourier{\y}/\fourier{\mathbf{h}}}$ ---where $\mathbf{h}$ is a discrete version of the filter $h$--- and the \emph{Wiener method}, which finds the optimal estimate of $x$ in the mean-square-error sense. Both methods assume a known filter and perform a division in the Fourier domain which can be unstable in practice. The first Bayesian take on the deconvolution problem can be attributed to Bretthorst in 1992, who proposed to use a multivariate normal (MVN) prior for the discrete-time case \cite{bretthorst1992bayesian}. \red{A year later, Rhode and Whittenburg addressed the Bayesian deconvolution of exponentially-decaying sinewaves \cite{rhode1993bayesian}, and then Sir David Mackay's book presented the Bayesian derivation of the deconvolution problem} assuming an MVN prior and connects it to other methods with emphasis on image restoration \cite[Ch.~46]{mackay2003information}. In the Statistics community, deconvolution refers to the case when $x$ and $h$ are probability densities rather than time series; this case is beyond our scope. 

%the need for continuous time priorrs
The use of different priors for Bayesian deconvolution follows the need to cater for specific properties of the latent time series $x$. For instance, \cite{adami2003variational} considered a mixture of Laplace priors which required variational approximations, while \cite{babacan2008variational} implemented a total variation prior with the aim of developing reliable reconstruction of images. However, little attention has been paid, so far, to the case when the source $x$ is a continuous time signal. Most existing approaches assume priors only for discrete-time finite sequences (or vectors) and, even though they can still be applied to continuous data by quantising time, their computational complexity explodes for unevenly sampled data or dense temporal grids. In the same manner, the kriging literature, which is usually compared to that of GPs, has addressed the deconvolution problem (see, e.g., \cite{jeulin1992practical,goovaerts2008deconv}) but has not yet addressed the continuous-time setting from a probabilistic perspective. \red{Additionally, recent advances of deconvolution in the machine learning and computer vision communities mainly focus on blending neural networks models with classic concepts such as prior design \cite{ren2020neural} and the Wiener filter \cite{NEURIPS2020_0b8aff04}, while still considering discrete objects (not continuous time) and not allowing in general for missing data.}

Convolution models have been adopted by the GP community mainly to 
 parametrise covariance functions for the multioutput \cite{boyle2005dependent,alvarez2009sparse,parra_tobar} and non-parametric cases \cite{tobar2015learning,bruinsma:2016}. \red{With the advent of deep learning, researchers have  replicated the convolutional structure of convolutional NNs (CNNs) on GPs. For instance, Deep Kernel Learning (DKL) \cite{wilson2016deep} concatenates a CNN and a GP so that the kernel of the resulting structure---also a GP---sports \textit{shared weights and biases}, which are useful for detecting  common features in different regions of an image and provides robustness to translation of objects. Another example are convolutional GPs
\cite{vanderwilk2017convolutional}, which extend DKL by concatenating the kernel with a \textit{patch-response function} that equips the kernel with convolutional structure without introducing additional hyperparameters. This concept has been extended to graphs and deep GPs---see \cite{walker2019graph,blomqvist2019deep} respectively. }

Though convolutions have largely aided the design of kernels for GPs, contributions in the "opposite direction", that is, to use the GP toolbox as a means to address the general deconvolution problem, are scarce. To the best of our knowledge, the only attempts to perform Bayesian deconvolution using a GP prior are works that either: focus specifically in detecting magnetic signals from spectropolarimetric observations \cite{ramos2015bayesian}; only consider discrete-time impulse responses \cite{tobar17a}; or, more recently, use an MVN prior for the particular case of a non-stationary Matérn covariance \cite{arjas2020blind}, \red{a parametrisation proposed by \cite{paciorek2006spatial} which has, in particular, been used for scatter radar data \cite{amt-15-3843-2022}.}

Building on the experimental findings of these works implementing deconvolution using GPs, our work focuses on the analysis and study of  kernels and filters, in terms of their ability to recover $x(t)$ from $\y$. The proposed strategy, termed GPDC, is expected to have superior modelling capabilities as compared to previous approaches, while having a closed form posterior deconvolution density which is straightforward to compute. However, there are aspects to be addressed before implementing GPDC, these are: i) when the integral in eq.~\eqref{eq:conv_proc} is finite in terms of the law of $x$, ii) when  $x$ can be recovered from $\y$, iii) how the blind scenario can be approached. Addressing these questions are the focus of  the following sections.

%!TEX root = ../GPDC.tex

\section{Analysis of the direct model} 
\label{sec:proposed_model}

\subsection[Assumptions on x and h]{Assumptions on  $x$ and $h$, and their impact on $f$}

We assume that $x$ is a stationary GP and its covariance kernel $K_x$ is integrable (i.e., $K_x\in  L_1$); this is needed for $x$ to have a well-defined Fourier power spectral density (PSD). Additionally, although the sample paths $x$ are in general not integrable, they are \emph{locally integrable} due to $K_x\in  L_1$, therefore, we assume that the filter $h$ decays \emph{fast enough} such that the integral in eq.~\eqref{eq:conv_proc} is finite. This is always obtained when either $h$ has compact support or when it is dominated by a Laplacian or a Square Exponential function. If these properties (integrability and stationarity) are met for $x$, they translate to $f$ via the following results.

\begin{remark} If $f(t)$ in eq.~\eqref{eq:conv_proc} is finite point-wise for any $t\in\R$, then $f(t) \sim \GP(0, K_f(t))$ with stationary kernel 
\begin{equation}
	K_f(t) = \int_{\R^2} h(\tau')h(\tau) K_x(\tau - (\tau'- t)) \dtau\dtau'.\label{eq:conv_var}
\end{equation}
\end{remark} 

\begin{lemma}
	\label{lemma:K_f}
	If the convolution filter $h$ and the %stationary 
	covariance $K_x$ are both integrable, then $K_f$ in eq. (\ref{eq:conv_var}) is integrable.
\end{lemma}

\begin{proof}
The integrability of $K_f(t)$ in eq.~\eqref{eq:conv_var} follows directly from applying Fubini Theorem and the triangle inequality (twice) to $\int_\R|K_f(t)|\dt$, to obtain $\int_\R|K_f(t)|\dt< \infty$ (full proof in the Appendix).
\end{proof}

%old proof
\iffalse
\begin{proof}
	This follows in the same vein as the standard proof of integrability of the convolution between two functions with a slight modification, since eq.~\eqref{eq:conv_var} comprises the composition of two convolutions rather than just one. Therefore, using Fubini Thm and the triangle inequality (twice), we have 
	\begin{align*}
	\int_\R|K_f(t)|\dt& = \int_\R\left| \int_{\R} h(\tau')\int_{\R}h(\tau) K_x(\tau - (\tau'-(t)) \dtau\dtau'\right|\dt &&\text{[Fubini on eq.~\eqref{eq:conv_var}]}\\ 
			& \leq \int_\R \int_{\R} |h(\tau')|  \int_{\R} \left|h(\tau)\right| \left|K_x(\tau - (\tau'-(t)) \right|\dtau\dtau'\dt&&\text{[triangle ineq. twice]}\\
			& =  \int_{\R} |h(\tau')|  \int_{\R} \left|h(\tau)\right|  \int_\R  \left|K_x(\tau - (\tau'-(t)) \right|\dt \dtau\dtau' &&\text{[Fubini]}\\
			& =  ||h||_1||h||_1 ||K_x||_1 && \\ 
			& < \infty.&& \text{[$h,K_x\in L_1$]}
	\end{align*}
\end{proof}
\fi

\begin{remark} 
Lemma \ref{lemma:K_f} provides a \textbf{sufficient} condition for the integrability of $K_f$, thus theoretically justifying i) a Fourier-based analysis of $f$, and ii) relying upon the GP machinery to address deconvolution from the lens of Bayesian inference. However, the conditions in Lemma \ref{lemma:K_f} are \textbf{not necessary}; for instance, if the  filter $h$ is the (non-integrable) Sinc function, the generative model in eqs.~\eqref{eq:source}-\eqref{eq:conv_y} is still well defined (see Example \ref{ex:RBF-Sinc}).
\end{remark}

\subsection{Sampling from the convolution model} 
\label{sub:sampling}

%At least 
We devise two ways of sampling from $f$ in eq.~\eqref{eq:conv_proc}: we could either i) draw a path from $x\sim\GP(0,K_x(t))$ and convolve it against $h$, or ii) sample directly from  $f\sim\GP(0, K_f(t))$. However, both  alternatives have serious limitations. The first one cannot be implemented numerically since the convolution structure implies that every element in $f$ depends on infinite values of $x$ (for a general $h$). The second alternative bypasses this difficulty by directly sampling a finite version of $f$, however, by doing so $x$ is integrated out, meaning that we do not know "to which" sample trajectory of $x$ the (finite) samples of $f$ correspond. 

The key to jointly sample finite versions (aka \emph{marginalisations}) of the source and convolution processes lies on the fact that, although the relationship between $x$ and $f$ established by eq.~\eqref{eq:conv_proc} is deterministic, the relationship between their finite versions becomes stochastic. Let us denote\footnote{Here, we use the compact notation $x(\t) = [x(t_1),\ldots,x(t_n)]$, where $\t =[t_1,\ldots,t_n]$.} finite versions of $x$ and $f$ by $\x = x(\t_x)$ and $\f = f(\t_f)$ respectively, where $\t_x\in\R^{N_x}$ and $\t_f\in\R^{N_f}$.  Therefore, we can hierarchically sample according to $p(\x,\f) = p(\f|\x)p(\x)$ in two stages: we first sample $\x\sim\MVN{0,K_x(\t_x)}$ and then $\f|\x\sim p(\f|\x)$ given by\footnote{We use the reversed-argument notation $K_{xf}(\t_x,\t_f) = K^\top_{fx}(\t_f,\t_x)$ to avoid the use of transposes. }
\begin{align}
	p(\f|\x) &= \MVN{\mu_{\f|\x},\sigma^2_{\f|\x}} \label{eq:samplig_dist}\\
	\mu_{\f|\x}&= K_{fx}(\t_f,\t_x)K_x^{-1}(\t_x)\x\nonumber\\
	\sigma^2_{\f|\x}&= K_f(\t_f) - K_{fx}(\t_f,\t_x)K_x^{-1}(\t_x)K_{xf}(\t_x,\t_f),\nonumber
\end{align}
where $K_{fx}(t_1,t_2) =  K_{fx}(t_1- t_2)$ is the stationary covariance between $f$ and $x$, given element-wise by
\begin{equation}
 	K_{fx}(t_1,t_2) = K_{fx}(t_1- t_2)= \int_\R h(t - (t_1-t_2))K_x(t)\dt. \label{eq:cross_var}
 \end{equation}
The integrability of ${K}_{fx}$ in eq. (\ref{eq:cross_var}) is obtained similarly to that of  $K_f$ in Lemma~\ref{lemma:K_f}, under the assumption that  $h,K_x\in L_1$.

Furthermore, let us recall that the conditional mean of $f(t)|\x$ is given by 
\begin{equation}
	\E{f(t)|\x}  =  \int_\R h(t + \tau)\E{x(\tau)|\x} \dtau. \label{eq:post_mean_f}
\end{equation}
This expression reveals that the expected value of $f(t)$ given $\x$ can be computed by first calculating the average interpolation of $x$ given $\x$, denoted $\E{x(\tau)|\x}$, and then applying the convolution to this interpolation as per eq.~\eqref{eq:conv_proc}. 

Additionally, let us also recall that the conditional variance of $f(t)|\x$ is given by 
\begin{equation*}
	\V{f(t)|\x} = \int_{\R^2} h(t+\tau)\V{x(\tau),x(\tau')|\x}h(t+\tau')\dtau\dtau'.
\end{equation*}
Observe that since the posterior variance of a GP decreases with the amount of observations,  $\V{x(\tau),x(\tau')|\x}$ approaches zero whenever $\t_x$ become more dense, and consequently so does $\V{f(t)|\x}$. Therefore, the more elements in $\x$, the smaller the variance  $\V{f(t)|\x}$ and thus the trajectories of $f|\x$ become concentrated around the posterior mean in eq.~\eqref{eq:post_mean_f}. This resembles a connection with the discrete convolution under missing source data, where one "interpolates and convolves", emphasising the importance of the interpolation (i.e., the prior) of $x$. Definition  \ref{def:SE_kernel} presents the Square Exponential (SE) kernel, and then Example  \ref{ex:sample} illustrates the sampling procedure and the concentration of $f|\x$ around its mean.
\red{

\begin{definition}\label{def:SE_kernel}
The stationary kernel defined as 
\begin{equation}
    K_{\text{SE}}(t) = \sigma^2\exp\left(-\frac{1}{2l^2} t^2\right)
\end{equation}
is referred to as Square Exponential (SE) and its parameters are magnitude $\sigma$ and lengthscale $l$. Alternatively, the SE kernel's lenghtscale can be defined in terms of its inverse lengthscale (or rate) $\gamma = \frac{1}{2l^2}$. The Fourier transform of the SE kernel is also an SE kernel, therefore, the support of the PSD of a GP with the SE kernel is the entire real line.
\end{definition} 

}

\begin{example}
	\label{ex:sample}
	Let us consider $K_x$ and $h$ to be SE kernels with lengthscale $l = \sqrt{0.05}$ (rate $\gamma = 10$), thus, $K_f$ and $K_{xf}$ are SE as well. Fig.~\ref{fig:sampling} shows $\x$ and $\f$ sampled over the interval $t\in[0,10]$: at the left (resp.~right) plot, we sampled a 40-dimensional (resp.~500-dimensional) vector $\x$ shown in red to produce a 1000-dimensional vector $\f$ shown in blue alongside the 95\% error bars for $p(f(t)|\x)$ in both plots. Notice how the larger  dimensionality of $\x$ over the fixed interval resulted in a tighter conditional density $p(f(t)|\x)$.
	\begin{figure*}
		\centering
		\includegraphics[width=0.8\textwidth]{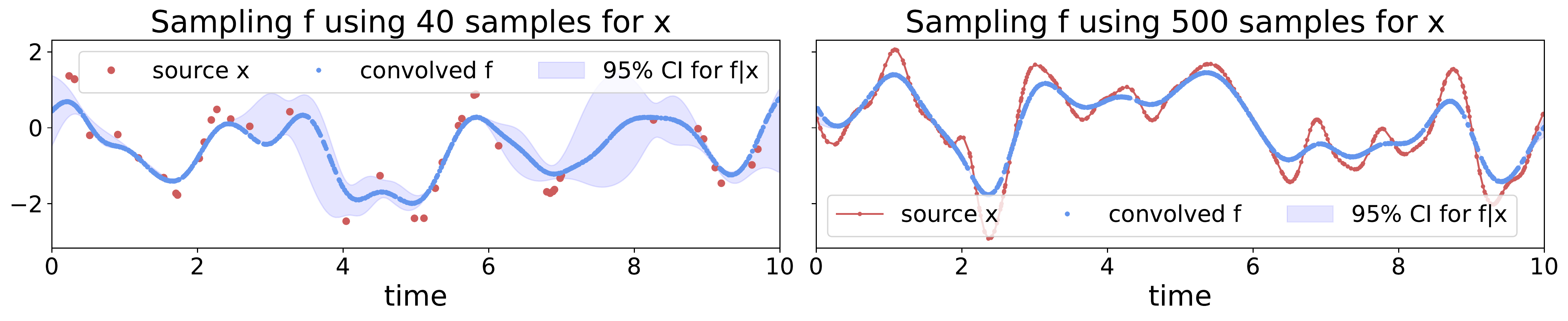}
		\caption{Hierarchical sampling $\x$ and $\f$ considering 40 (left) and 500 (right) samples for $\x$ in the interval [0,10]. Both $K_x$ and $h$ (and consequently $K_{xf}$ and $K_f$) are square-exponential kernels \red{with lengthscale $l = \sqrt{0.05}$ (rate $\gamma = 10$) }.}
		\label{fig:sampling}
	\end{figure*} 
\end{example}

\begin{remark}
The GP representation of the convolution process allows us to generate finite samples $\x$, $\f$ (over arbitrary inputs) of the infinite-dimensional processes $x$ and $f$ respectively. This represents a continuous and probabilistic counterpart to the classic discrete-time convolution that i) can be implemented computationally and, ii) is suitable for real-world scenarios where data may be non-uniformly sampled. 
\end{remark}

%!TEX root = ../GPDC.tex

\section{The inverse problem: Bayesian nonparametric deconvolution} 
\label{sec:deconvolution}

We now turn to the inverse problem of recovering $x$ from the finite vector $\y = [y_1,\ldots,y_N]$ observed  at times  $\t = [t_1,\ldots,t_N]$. In our setting defined in eqs.~\eqref{eq:source}-\eqref{eq:conv_y}, the deconvolution $x|\y$ is a GP given by
\begin{align}\label{eq:deconv}
x|\y &\sim \GP(\mu, \sigma^2)\\
	\mu(t) &= K_{xy}(t, \t) K_\y^{-1}\y, \label{eq:deconv_mean}\\
\hspace{-0.85em}\sigma^2 (t_1,t_2) &= K_x(t_1,t_2) - K_{xy}(t_1, \t) K_\y^{-1}K_{yx}(\t, t_2)\label{eq:deconv_cov},
\end{align}
where $ K_\y = K_f(\t,\t) + \eye_N\sigma_n^2$ is the marginal covariance matrix of the observations vector $\y$, $\eye_N$ is the $N$-size identity matrix, and $K_{xy}(t_1,t_2)$ denotes the cross-covariance between $x(t_1)$ and $y(t_2)$, which is in turn equal to $K_{xf}(t_1,t_2)=K_{xf}(t_1-t_2)$.

\subsection{When is deconvolution possible?}
\label{sec:deconv_possible}
%We are interested in answering the question  Specifically, 
We aim to identify the circumstances under which $\y$ provides \emph{enough information} to determine $x$ with as little uncertainty as possible.
To this end, we consider a windowed (Fourier) spectral representation of  $x|\y$---the GP in eq.(\ref{eq:deconv})---given by 
\begin{equation}
	\hat x_w (\xi) = \fourier{w(t)x(t)}(\xi) = \int_\R w(t)x(t)e^{-j2\pi\xi t}\dt,\label{eq:spec_rep}
\end{equation}
where $\fourier{\cdot}$ denotes the continuous-time Fourier transform (FT) operator, $\xi\in\R$ is the frequency variable, and the window $w:\R\to\R$ has compact support (such that the integral above is finite). This windowed spectral representation is chosen since the GP $x|\y$ is nonstationary (its power spectral density is not defined), and the sample trajectories of  $x|\y$ following the distribution $\GP(\mu, \sigma^2)$ are not Lebesgue integrable in general (their FT cannot be computed). Based on the representation introduced in eq.~\eqref{eq:spec_rep}, we define the concept of \emph{successful recovery} and link it to the spectral representation of $K_x$ and $h$.

\begin{definition} 
	\label{def:recover}
	We say that $x(\cdot)$ can be \textbf{successfully recovered} from observations $\y=[y_1,\ldots,y_N]$ if, $\forall\xi\in\R$, the spectral representation $\hat x_w (\xi)$ in eq.~\eqref{eq:spec_rep} can be recovered at an arbitrary precision by increasing the amount of observations $N$, for an arbitrary compact-support window $w(\cdot)$.
\end{definition} 

\begin{theorem} 
\label{teo:recover}
For the setting defined in eqs.~\eqref{eq:source}-\eqref{eq:conv_y}, a necessary condition to successfully recover $x$ from $\y$ is that the spectral support of $h$ contains that of $K_x$, that is, $\supp{\fourier{K_x}}\subseteq\supp{\fourier{h}}$. 
\end{theorem}
\begin{proof}
	From eqs.~\eqref{eq:source}-\eqref{eq:conv_y} and linearity of the FT, the windowed spectrum $\hat x_w (\xi)$ is given by a complex-valued\footnote{For complex-valued GPs, see \cite{boloix2018complex,ambrogioni2019complex,icassp15}.} GP, with mean and marginal variance (calculations in the Appendix):
	\begin{align}
	  	\mu_{\hat x_w }(\xi) &= \left(\hat m_x(\xi) + \red{\hat K_x(\xi)\hat h(\xi)}e^{-j2\pi\t\xi}K_\y^{-1}\y\right)\star \hat w(\xi),\\
	  	\sigma^2_{\hat x_w }(\xi) &= \blue{\hat{K}_x(\xi)}\star|\hat{w}(\xi)|^2 \nonumber\\
	  	&\quad\quad  -  \| \hat w(\xi)   \star  \red{\hat K_x(\xi)\hat h(\xi)}e^{-j2\pi\xi\t}  \|_{ K_\y^{-1}},\label{eq:post_var_xw}
	  \end{align}  
	 where $\|\v\|_A = \v^\top A \v$ denotes the Mahalanobis norm of $\v$ w.r.t. matrix $A$, and $\hat h = \fourier{h}$ was used as a compact notation for the FT of $h$. Following Definition \ref{def:recover}, successful recovery is achieved when the posterior  variance of $\hat x_w (\xi)$ vanishes, i.e., when the two terms at the right hand side of eq.~\eqref{eq:post_var_xw} cancel one another. For this to occur,  it is necessary that these terms have the same support, however, as this should happen for an arbitrary window $\hat w(\xi)$ we require that $\supp{\blue{\hat K_x(\xi)}} = \cup_{i=1}^N\supp\red{\hat K_x(\xi)\hat h(\xi)}e^{-j2\pi\xi t_i}$. Since $\cup_{i=1}^N\supp e^{-j2\pi\xi t_i}=\R$, what is truly needed is that $\supp{\blue{\hat K_x(\xi)}} = \supp{\red{\hat K_x(\xi)\hat h(\xi)} }$ which is obtained when $\supp{\hat K_x (\xi)}\subseteq\supp{\hat h (\xi) }$.
\end{proof}
\begin{remark}\label{rem:sufficient_obs} Theorem \ref{teo:recover} gives a \textbf{necessary} condition to recover $x$ from $\y$ in terms of how much of the spectral content of $x$, represented by $\hat K_x$, is not suppressed during the convolution against $h$ and thus can be extracted from $\y$. A stronger \textbf{sufficient} condition for successful recovery certainly depends on the locations where $\y$ is measured. Intuitively, recovery depends on having enough observations, and on the vector $e^{-j2\pi\xi \t}$ being aligned with the eigenvectors of  $K_y^{-1}(\t,\t)$ defining the norm in eq.~\eqref{eq:post_var_xw}; thus resembling the sampling theorem in \cite{Shannon_1949,Nyquist_1928}. 
\end{remark}

Definition \ref{def:sinc_kernel} introduces the Sinc kernel \cite{tobar19b}. Then, Example \ref{ex:RBF-Sinc} illustrates the claims in Theorem \ref{teo:recover} and Remark \ref{rem:sufficient_obs}, where GPDC recovery is assessed in terms of spectral supports and amount of observations. 

\red{

\begin{definition}\label{def:sinc_kernel}
The stationary kernel defined as 
\begin{equation}
    K_{\text{Sinc}}(t) = \frac{\sigma^2 \sin( \Delta\pi t)}{ \Delta \pi t}
\end{equation}
is referred to as the (centred) Sinc kernel and its parameters are magnitude $\sigma$ and width $\Delta$. Though the Sinc function is not integrable (see Lemma \ref{lemma:K_f}), it admits a Fourier transform which takes the form of a rectangle of width $\Delta$ centred in zero. Therefore, a GP with a Sinc kernel does not generate paths with frequencies greater than $\Delta/2$. 
\end{definition} 

}

	\begin{figure*}
		\centering
		\includegraphics[width=0.48\textwidth]{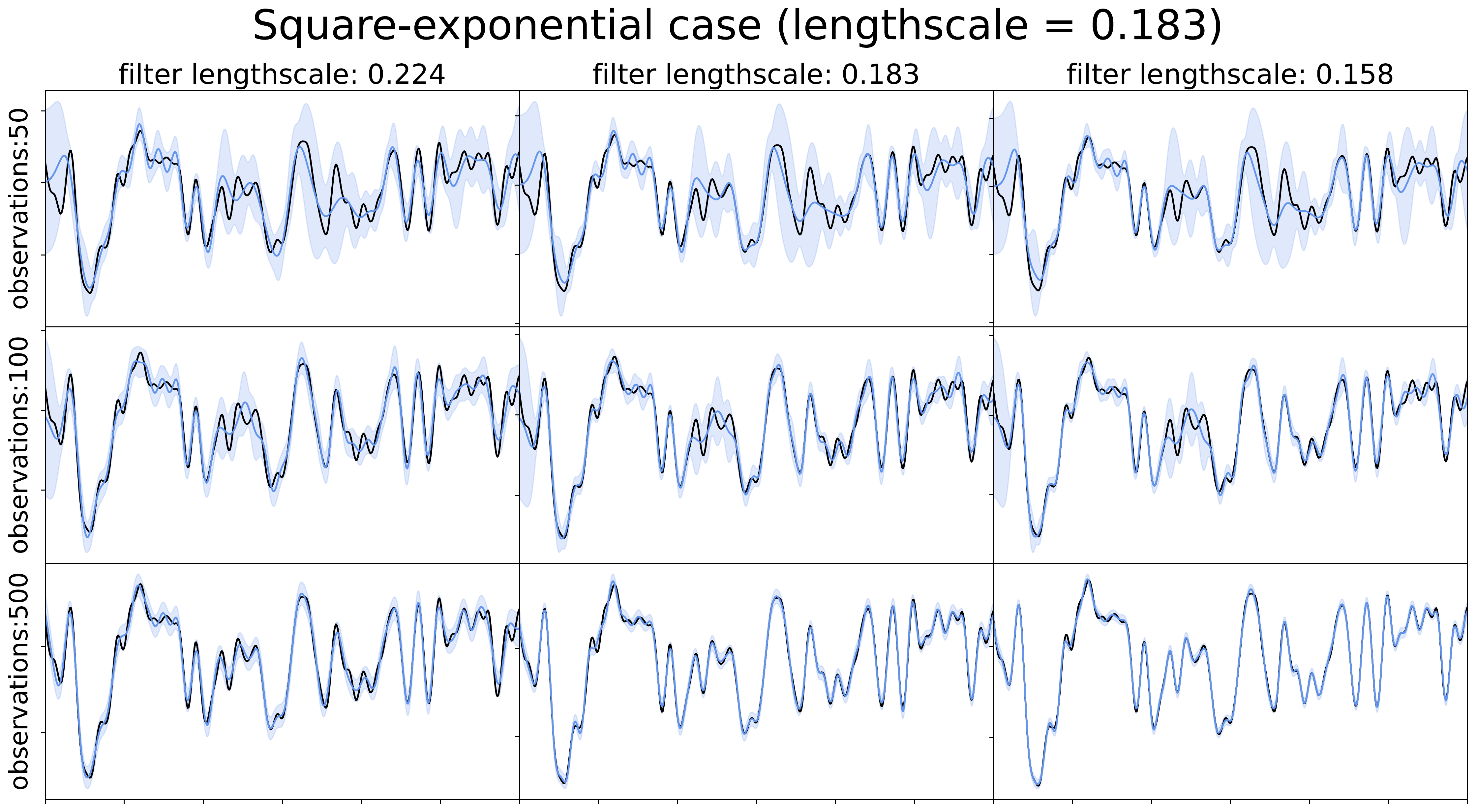}\hfill
		\includegraphics[width=0.48\textwidth]{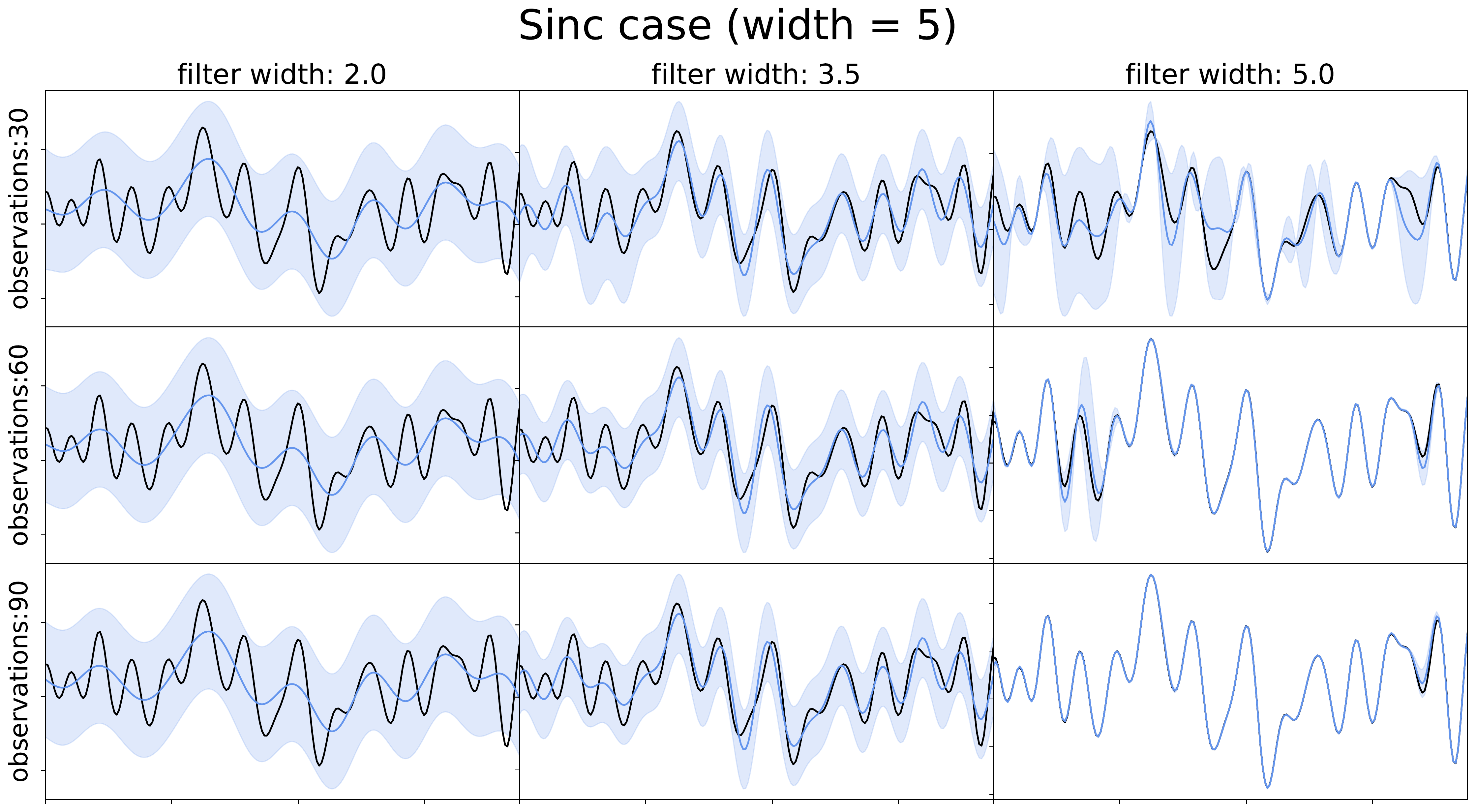}
		\caption{Implementation of GPDC when both $h$ and $K_x$ are SE (left) or Sinc (right). In each case, we considered increasing number of observations $\y$ (top to bottom) and different parameters for $h$ (left to right)---\red{the values for all the width and lenghtscale parameters are presented in the figure titles}. The plots show the ground truth source $x$ (black), the posterior GPDC deconvolution mean (blue) and GPDC's 95\% error bars (light blue). In line with Theorem \ref{teo:recover} and Remark \ref{rem:sufficient_obs}, notice how in the SE case (left) more observations improve the recovery, whereas in the Sinc case (right) recovery is not possible even for increasing amount of observations when the filter is too narrow (first and second columns).}
		\label{fig:synth_deconv}
	\end{figure*}
	\begin{example}
\label{ex:RBF-Sinc}
	Let us consider two scenarios for GPDC: In the first one, both $K_x$ and $h$ are Square Exponential  (SE) kernels and therefore $\supp \fourier{K_x} =  \supp\fourier {h} = \R$. In the second case, both $K_x$ and $h$ are Sinc kernels \cite{tobar19b}, and the overlap of their PSDs depends on their parameters. We implemented GPDC in both cases for different amount of observations and parameters, in particular, we considered the scenario where $\supp \hat K_x\not\subset \supp\hat h$ for the Sinc case. Fig.~\ref{fig:synth_deconv} (left) shows how in the SE case an increasing number of observations always improves the recovery regardless of the parameters of $h$ and $K_x$. On the contrary, notice that for the Sinc case (Fig.~\ref{fig:synth_deconv}, right), even with a large number of observations, the latent source $x$ cannot be recovered. This is due to certain spectral components of $x$ being removed by the \emph{narrowbandness} of the convolution filter $h$. \red{All magnitude parameters were set one, see Fig.~\ref{fig:synth_deconv} for the width (Sinc) and lenghtscale parameters (SE)}.

\end{example} 

\section{The blind deconvolution case }
\label{sec:deconvolution-blind}
\subsection{Training and observability}
\label{sec:training}

\red{Implementing GPDC requires choosing the hyperparameters, i.e., the parameters of $K_x$ and $h$, and the noise variance $\sigma^2$. Depending on the application the filter might be known beforehand, for instance, in de-reverberation \cite{naylor2010speech} $h$ is given by the geometry of the room, whereas in astronomy $h$ is given by the atmosphere and/or the telescope \cite{starck2002deconvolution,tobar21b}. When $h$ is unknown, a case referred to as \emph{blind deconvolution}, its parameters have to be learnt from data alongside the other hyperparameters. }

\red{We fit the proposed model---see eqs.~\eqref{eq:source}-\eqref{eq:conv_y}---by maximising the log-likelihood  
\begin{equation}
	l(K_x,\sigma^2,h) = -\frac{n}{2}\log 2\pi - \frac{1}{2}\log\det K_\y - \frac{1}{2}\y^\top K_\y^{-1}\y, \label{eq:likelihood}
\end{equation}
where $K_\y = K_f(\t,\t) + \eye_N$, with $\eye_N$ the identity matrix of size $N$, and the relationship between $K_f$ and $K_x$ follows from eq.~\eqref{eq:conv_var}. Therefore, $K_x$, $h$ and $\sigma^2$ appear in $l$ in eq.~\eqref{eq:likelihood} through $K_\y$.}

Maximum likelihood, however, might not recover all hyperparameters in a unique manner, since $h$ and $K_x$ are \emph{entangled} in $K_\y$ and thus can be unidentifiable. For instance, if both $h$ and $K_x$ are SEs (as in Example \ref{ex:sample}) with lengthscales $l_h$ and $l_x$ respectively, then $K_\y$ is also an SE with lenghtscale $l_y = l_x + l_h$; therefore, the original lengthscales cannot be recovered from the learnt $l_y$. More intuitively, the unobservability of the convolution can be understood as follows: a given signal $f$ could have been generated either i) by a fast source $x$ and a wide filter $h$, or ii) by a slow source and a narrow filter. Additionally, it is impossible to identify the temporal location of $h$ only from $\y$, since the likelihood is insensitive to time shifts of $x$ and $h$. Due to the symmetries in the deconvolution problem learning the hyperparameters should be aided with as much information as possible about $h$ and $K_x$, in particular in the \emph{blind} scenario. 

\subsection[A tractable approximation of h]{A tractable approximation of $h$}
\label{sec:discrete_h}

Closed-form implementation of GPDC \emph{only} depends on successful computation of the integrals in eqs.~\eqref{eq:conv_var} and \eqref{eq:cross_var}. Those integrals can be calculated analytically for some choices of $K_x$ and $h$, such as the \red{SE or Sinc kernels in Definitions \ref{def:SE_kernel} and \ref{def:sinc_kernel} respectively}, but are in general intractable. Though it can be argued that computing these integrals might hinder the general applicability of GPDC, observe that when the filter $h$ is given by a sum of Dirac deltas $\left\{w_i \right\}_{i=1}^M$ at locations $\{l_i\}_{i=1}^M$, i.e.,
\begin{equation}
 	h(t) = \sum_{i=1}^M w_i\delta_{l_i}(t)\label{eq:discrete_h},
 \end{equation} 
 the covariance $K_f$ and the cross-covariance $K_{fx}$ turn into summations of kernel evaluations:
 \begin{align}
 	K_f(t) &= \sum_{i=1}^M\sum_{j=1}^M w_iw_jK_x(l_i - (l_j-t))\label{eq:Kf-h}\\
 	K_{xf}(t_1,t_2) = K_{xf}(t_1- t_2) &= \sum_{i=1}^M w_i K_x(l_i - (t_1-t_2))\label{eq:Kxf-h}.
 \end{align}

The discrete-time filter $h$ is of interest in itself, mainly in the digital signal processing community, but it is also instrumental in approximating GPDC for general applicability. This is because when the integral forms in $K_f,K_{xf}$, in eqs.~\eqref{eq:conv_var} and \eqref{eq:cross_var}, cannot be calculated in closed-form, they can be approximated using different integration techniques such as quadrature methods, Monte Carlo or even importance sampling (see detailed approximations in the Appendix). 

\iffalse
 importance sampling (IS) using a set of
\begin{align}
	K_f(t) &\approx\frac{1}{M^2}\sum_{i,j=1}^M\frac{h(l_i)}{\phi(l_i)}\frac{h(l_j)}{\phi(l_j)}K_x(l_i-(l_j-t)), \label{eq:Kf-h-app}\\
K_{xf}(t) &\approx \frac{1}{M}\sum_{i=1}^M\frac{h(l_i)}{\phi(l_i)}K_x(l_i-t),\label{eq:Kxf-h-app}
\end{align}
where the density $\phi(\cdot)$ is required to \emph{dominate} $|h(\cdot)|$ in the sense that $\phi(\tau) =0 \Rightarrow h(\tau) = 0$. 
\begin{remark}\label{rem:aproxx_h} Notice from eqs.~\eqref{eq:Kf-h}-\eqref{eq:Kxf-h} and eqs.~\eqref{eq:Kf-h-app}-\eqref{eq:Kxf-h-app}, that the IS approximation results in the (exact) covariances corrresponding to a discrete filter $h$---as in eq.~\eqref{eq:discrete_h}---with weights $w_i = h(l_i)(M\phi(l_i))^{-1}, l_i\sim\phi$. Therefore, as the IS approximation converges, due to the Law of Large Numbers, to the true integrals as $1/\sqrt{M}$, the deconvolution equations can be approximated, to an arbitrary degree of accuracy, by increasing $M$ thus resulting in a tractable GPDC approximation.
\end{remark}
\fi

In terms of training (required for blind deconvolution) another advantage of the discrete-time filter is that its hyperparameters (weights $\left\{w_i \right\}_{i=1}^M$) do not become entangled into $K_\y$, but rather each of them appear bilinearly in it---see eq.~\eqref{eq:Kf-h}. This does not imply that parameters are identifiable, yet it  simplifies the optimisation due to the bilinear structure. Lastly, notice that since the sample approximations of $K_y$ (e.g., using quadrature or Monte Carlo) are of low order compared with the data, they do not  contribute with critical computational overhead, since evaluating eq.~\eqref{eq:likelihood} is still dominated by the usual GP cost of $\mathcal{O}(N^3)$. \red{Definitions \ref{def:SM_kernel} and \ref{def:triangular_filter} present the  Spectral Mixture kernel and the triangular filter respectively, then Example \ref{ex:synth-h-discrete} illustrates the discretisation procedure for training GPDC in the blind and non-blind cases as described above.} 

\red{

\begin{definition}\label{def:SM_kernel}
The stationary kernel defined as 
\begin{equation}
    K_{\text{SM}}(t) = \sigma^2\exp\left(-\frac{1}{2l^2} t^2\right)\cos\left(2\pi\nu t\right)
\end{equation}
is referred to as the (single component) Spectral Mixture  (SM) and its parameters are magnitude $\sigma$, lengthscale $l$ and frequency $\nu$. Akin to the SE kernel in Def.~\ref{def:SE_kernel}, the SM kernel's lenghtscale can be defined via its rate $\gamma = \frac{1}{2l^2}$. The Fourier transform of the SM kernel is an SE kernel centred in frequency $\nu$.
\end{definition}

\begin{definition}\label{def:triangular_filter}
The function defined as 
\begin{equation}
    h_{\text{Tri}}(t) = \sigma^2\max\left(1-\frac{2|t|}{\Delta},0\right)
\end{equation}
is referred to as the triangular filter, also known as the Bartlett window, and its parameters are magnitude $\sigma$ and width $\Delta$. 
\end{definition} 

}

	\begin{figure}
	\centering
	  	\includegraphics[width=.75\textwidth]{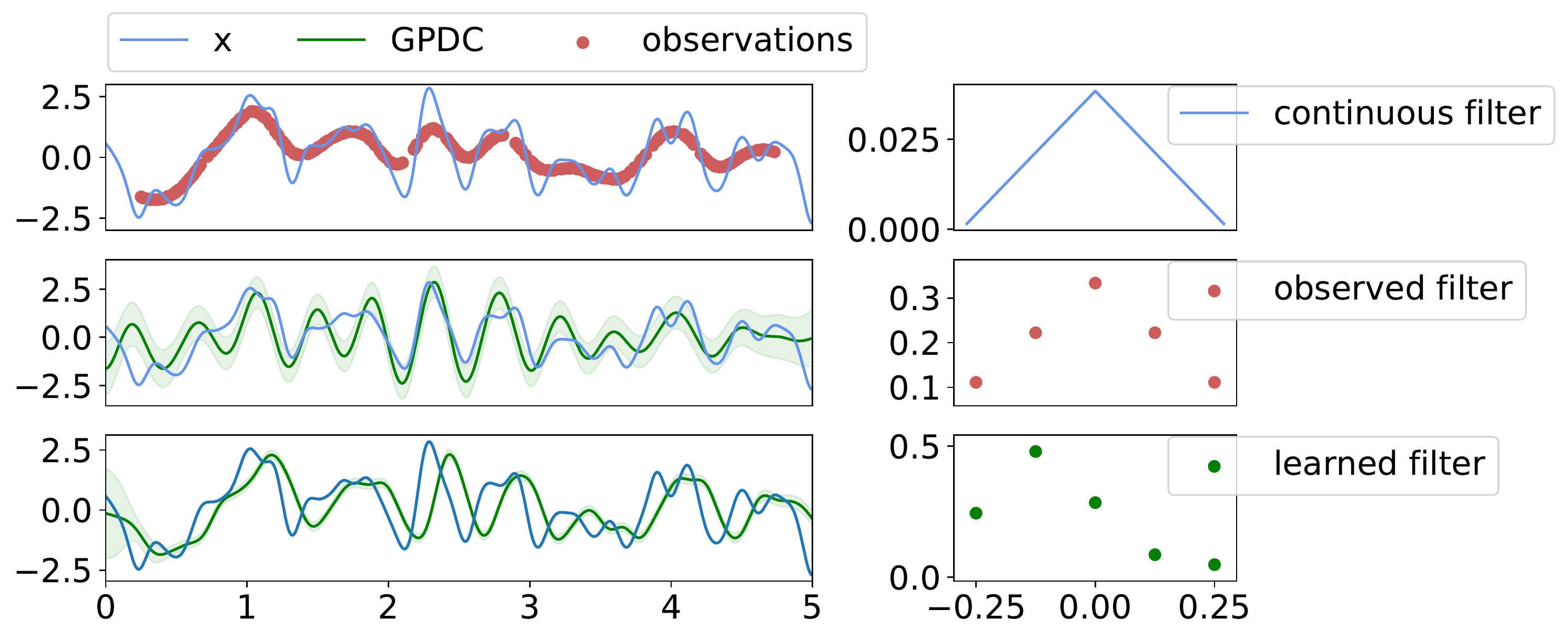}
	\caption{GPCD with a discrete-approximation filter $h$. Top: ground truth source $x$, observations $\y$ and filter $h$ \red{The ground truth $x$ is drawn from a GP with a spectral  mixture kernel $K_x$ (see Def.~\ref{def:SM_kernel}) with parameters $\sigma=1$, $\gamma=50$, $\nu=1$, $\sigma_n=0.01$} . Middle: GPDC (green) using a discrete filter calculated from the true $h$ only $K_x$ was learnt, \red{our method recovered $\gamma=11.8$, $\nu=1.9$ and $\sigma_n=0.03$ and the magnitude $\sigma$ was fixed to one due to unobservability arising from the multiplication of $h$ and $K_x$}. Bottom: posterior deconvolution in the blind case: all hyperparameters, $h$, $K_x$ and $\sigma^2$, are learnt via maximum likelihood. The learnt filter is plotted in the bottom right corner (green dots) and the recovered parameters were $\gamma=38.8$, $\nu=1e^{-6}$ and $\sigma_n=0.01$ ($\sigma$ still fixed to one). Observe how, despite the seemingly poor recovery of the parameter values, the deconvolution in the blind case (bottom-left plot, green) was correctly detected only up to an unidentifiable temporal due to the unobservability issues mentioned above. }
	\label{fig:ex_discrete_h}
	\end{figure}

\iffalse
	\begin{figure*}[t]
\floatbox[{\capbeside\thisfloatsetup{capbesideposition={left,top},capbesidewidth=4cm}}]{figure}[\FBwidth]
{\caption{GPCD with a discrete-approximation filter $h$. Top: ground truth source $x$, observations $\y$ and filter $h$. Middle: GPDC (green) using a discrete filter calculated from the true $h$ (only $K_x$ is learnt). Bottom: posterior deconvolution in the blind case: all hyperparameters, $h$, $K_x$ and $\sigma^2$, are learnt via maximum likelihood.}\label{fig:ex_discrete_h}}
{
\includegraphics[width=.7\textwidth]{img/learning_weights_spectral_0.pdf}
}
\end{figure*}
\fi

\begin{example}\label{ex:synth-h-discrete}
	Let us consider the case where $K_x$ is an SE kernel (see Def.~\ref{def:SE_kernel}) and $h$ is a triangle (see Def.~\ref{def:triangular_filter}). We avoid computing the integrals for $K_f$ and $K_{xf}$, and consider two scenarios: i) learn discrete approximations both for $h$ and $K_x$ via maximum likelihood as in Sec.~\ref{sec:deconvolution-blind}\ref{sec:training}, ii) fix a discrete filter according to the true $h$, and only train $K_x$ as in Sec.~\ref{sec:deconvolution-blind}\ref{sec:discrete_h}. Fig.~\ref{fig:ex_discrete_h} shows these implementations, notice how the discrete filter (only 5 points) allows for a reasonable recovery of the latent source  $x$ from $\y$. Furthermore, according to the unobservability of $h$, there is an unidentifiable lag between the true source and the deconvolution for the  blind case.

\end{example}

	\iffalse %old_fig
	\begin{figure*}
	\adjustbox{valign=c}{%
	\begin{minipage}[t]{.7\linewidth}
	\caption{GPCD with a discrete-approximation filter $h$. Top: ground truth source $x$, observations $\y$ and filter $h$. Middle: GPDC (green) using a discrete filter calculated from the true $h$ (only $K_x$ is learnt). Bottom: posterior deconvolution in the blind case: all hyperparameters, $h$, $K_x$ and $\sigma^2$, are learnt via maximum likelihood.}
		 \label{fig:ex_discrete_h}
	\end{minipage}}%
	\adjustbox{valign=c}{%
	\begin{minipage}[t]{.65\linewidth}
	  \begin{subfigure}{\linewidth}
	  \centering
	  	\includegraphics[width=.95\textwidth]{img/learning_weights_spectral_0.pdf}
	  %\caption{test subfigure one}
	  \end{subfigure}%\par\bigskip
	\end{minipage}}
	\end{figure*}
	\fi %old_fig

\iffalse
\begin{example}
	Let us consider the case where both $K_x$ and $h$ are Gaussian RBF (aka square-exponential) kernels, in which case, both $K_f$ and $K_{xf}$ can be computed in closed form and are RBFs as well. We then consider a pair $x,f$ over the time interval $t\in\{0,10\}$ and sample 200 (noisy) observations of $y$. Fig.~\ref{fig:synth} shows the generated signals at the top and the reconstruction at the bottom. Notice how the proposed model was able to recover high-frequency content that was seemingly removed from in the convolution.
	\begin{figure}
		\centering
		\includegraphics[width=0.8\textwidth]{img/synth_x_y.pdf}
		\includegraphics[width=0.8\textwidth]{img/synth_x_xhat.pdf}
		\caption{Illustrative example of GP-deconvolution using synthetic data. Top: signals generated using RBF kernels, top: GP-deconvolution (posterior mean and error bars) alongside ground truth.}
		\label{fig:synth}
	\end{figure}
\end{example}

Other ideas:
\begin{itemize}
	\item do we need a posterior over $h$, can we do with a Laplace approximation?
	\item \emph{telescopic} formulation: every GP is a convolved GP from white noise: until when do we deconvolve. 
	\item find the approximation error/rate of convergence of using discrete $h$ (hopefully). Evaluate sensibility of the posterior variance as a function of the (approximation) stepsize of $h$.
	\item experiments
\end{itemize}

\fi

%!TEX root = ../GPDC.tex

\section{Experiments}
\label{sec:exp}

We tested GPCD in two experimental settings. The first one simulated an acoustic de-reverberation setting and its objective was to validate the point estimates of GPDC in  the noiseless, no-missing-data, case (i.e., $\sigma_n^2=0$) against standard deconvolution methods. The second  experiment dealt with the recovery of latent images from low-resolution, noisy and partial observations, both when the filter $h$ is known or unknown (a.k.a. blind superresolution). The code used for the following experiments and the examples in the paper can be found in \url{https://github.com/GAMES-UChile/Gaussian-Process-Deconvolution}.

\begin{figure}[t]
\centering
        \includegraphics[width=0.75\textwidth]{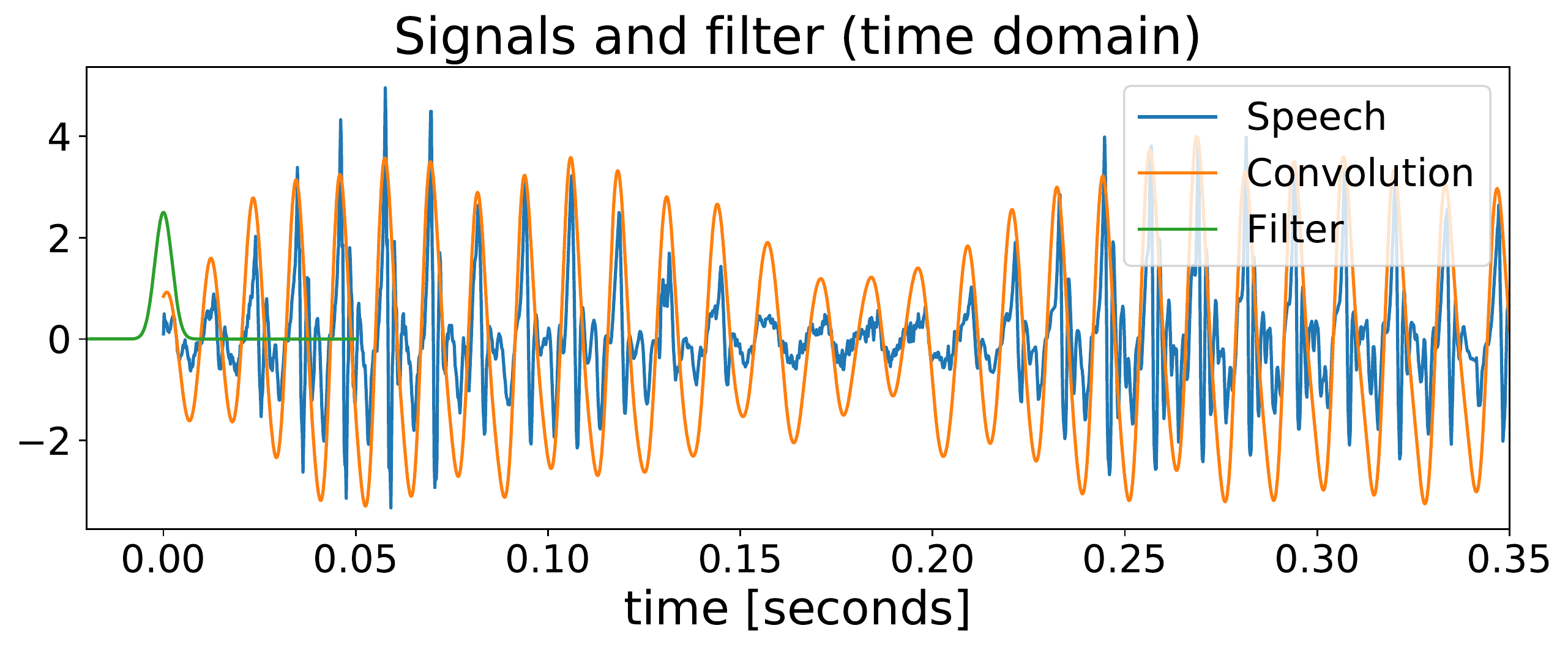}\\
        \includegraphics[width=0.75\textwidth]{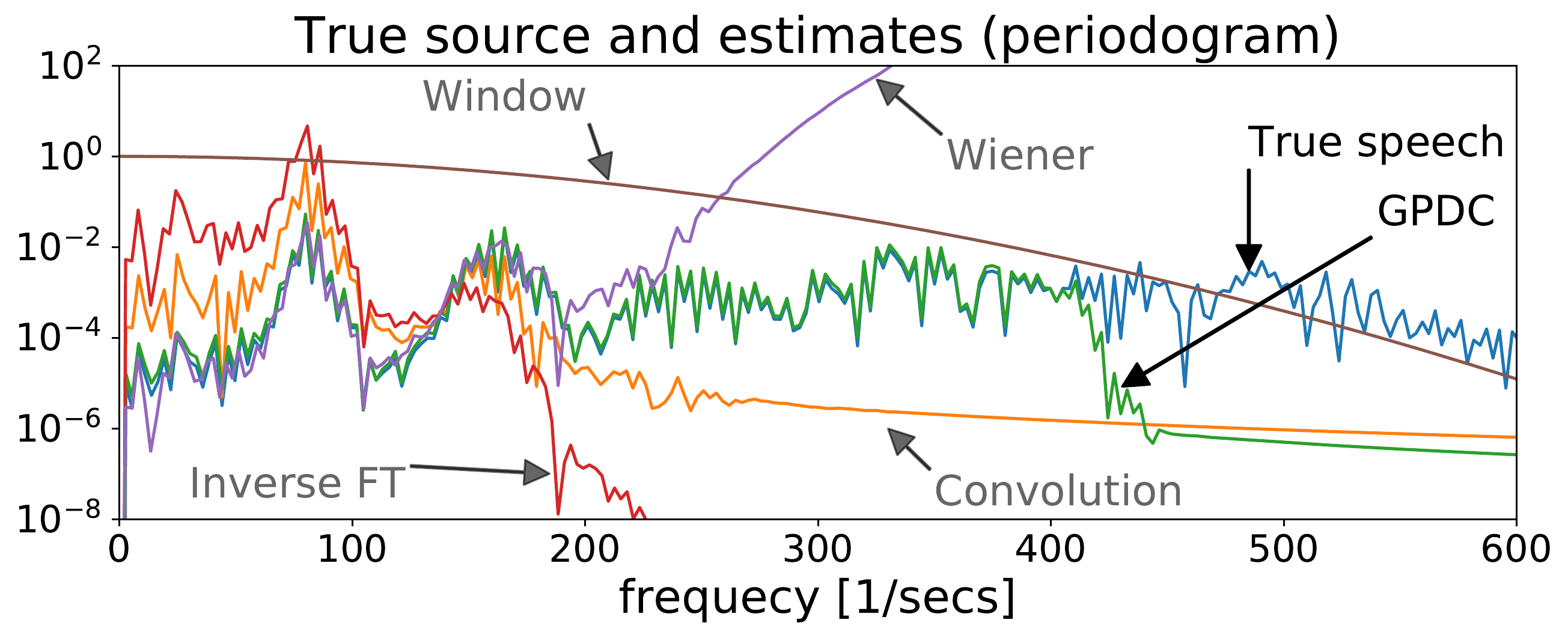}
        \caption{Recovering a speech signal using GPDC, Wiener and the Inverse FT method. Signals are shown in the temporal domain (top) and estimates in the frequency domain (bottom).}
        \label{fig:reverb} 
\end{figure}

\subsection{Acoustic deconvolution.} We considered a 350 [ms] speech signal (from \url{http://www.mcsquared.com/reverb.htm}), \red{ re-sampled at 5512.5 hertz (i.e., 2000 samples), standardised it and then convolved it against an SE filter of lenghtscale $l=2.2$ [ms]}. Fig.~\ref{fig:reverb} (top) shows the speech signal, the filter and the convolution. \red{ Since speech is known to be piece-wise stationary, we implemented GPDC using an SE kernel, then learnt the hyperparameters via maximum likelihood as explained in Section 5\ref{sec:training}. The learnt hyperparameters were $\sigma = 0.95, l = 5\cdot10^{-4}, \sigma_{\text{noise}} = 0.12$. Notice that these values are consistent with the data pre-processing (the variances sum approximately one since the signal was standardised) and from what is observed in Fig.~\ref{fig:reverb} (top) regarding the lengthscale.} 

The proposed GPDC convolution (assuming a known filter $h$ given by an SE kernel presented above) was compared  against the \emph{Wiener} deconvolution and the \emph{Inverse FT} methods mentioned in Sec.~\ref{sec:background}. Fig.~\ref{fig:reverb} (bottom) shows the periodogram of the data alongside Wiener, Inverse FT and GPDC, and the filter $h$. Observe how GPCD faithfully followed the true latent source across most the spectral range considered, whereas Wiener diverged and the Inverse FT decayed rapidly perhaps due to their numerical instability. It is worth noting that the GPDC estimate ceases to follow the true source precisely at the frequency when the filter decays, this is in line with Thm.~\ref{teo:recover}. Table \ref{table:audio_perf} shows the performance of the methods considered in time and frequency, in particular via the discrepancy between the true power spectral density (PSD) and its estimate using the MSE, the KL divergence and Wasserstein distance \cite{tobar20c} (the lower the better). Notice that GPDC outperformed Wiener and inverse FT under all metrics. The estimates in the temporal domain are shown in additional figures in the Appendix.

 \begin{figure*}[t]
  \centering
    \includegraphics[width=0.99\textwidth]{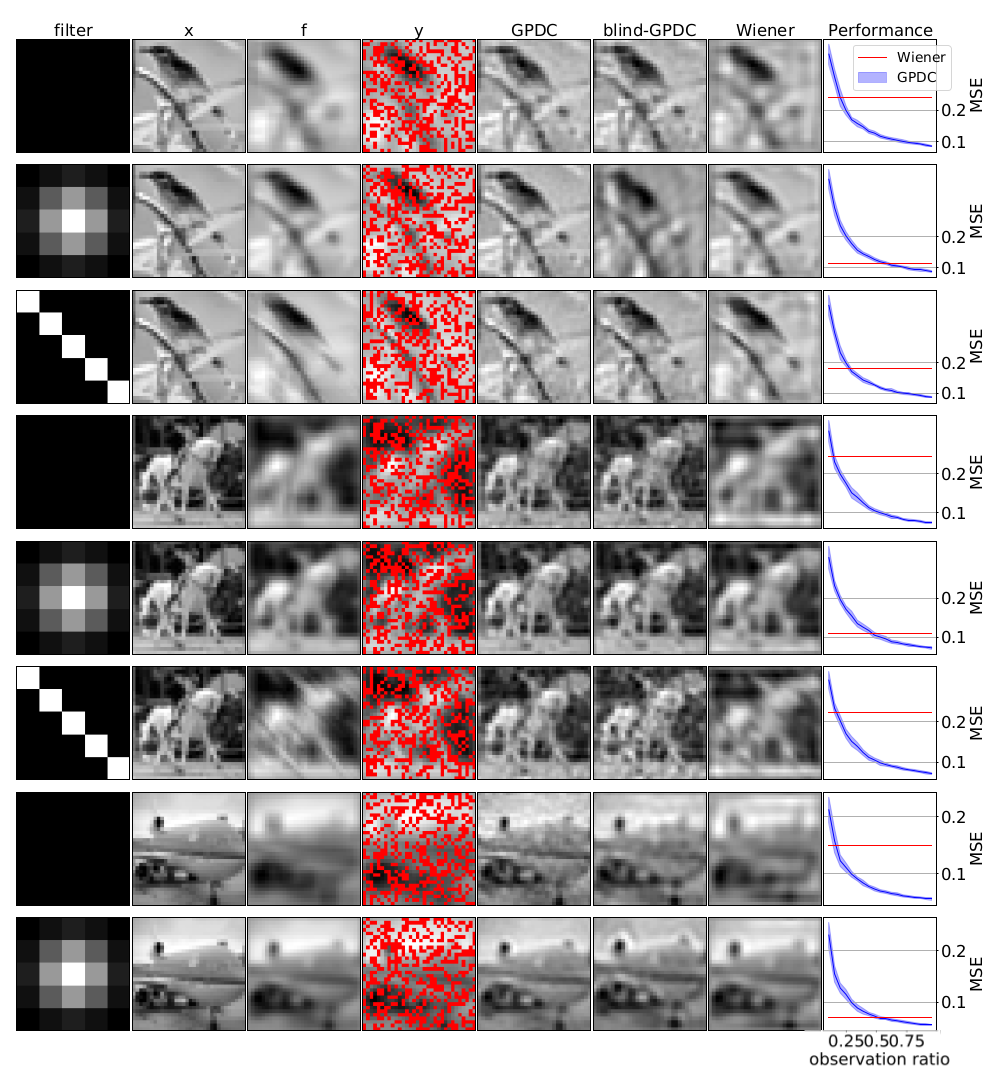}
    \caption{Deconvolution of 8 images from lower-resolution, noisy and incomplete observations. These 8 images were created from 3 source images to which we applied one of the five described filters. Left to right: filter used $h$, true  image  $x$, convolution $f$, observation $\y$ ($25\%$ missing pixels in red), GPDC estimate, blind-GPDC estimate and Wiener estimate. Rightmost plots show the performances of GPDC (blue) and Wiener filter (red) versus the number of observed pixels (2-standard-deviation error bars computed over 20 trials for GPDC).}
\label{fig:CIFAR}
\end{figure*}

\begin{table}[t]
%\scriptsize
%\hspace{1.5em}
\caption{Quantitative evaluation of the audio de-reverberation experiment in time and frequency}
\label{table:audio_perf}
\vspace{0.5em}
\centering
\small
\begin{tabular}{l|ccc}
       metric    & GPDC            & Wiener & inv-FT \\ \hline
MSE (time) & \textbf{19.0}   & 35.0   & 41.9   \\
MSE (PSD) & \textbf{0.015}  & 0.058  & 0.153  \\

Kullback–Leibler (PSD)  & \textbf{0.05}   & 0.20   & 0.45   \\
Wasserstein (PSD)  & \textbf{2124.3} & 3643.8 & 4662.6 \\
\end{tabular}
\end{table}

%Show how unstable is to deconvolve using \cite{stefanski1990deconvolving} 

\subsection{Blind image super-resolution.} For a $32\times 32$ image $x$, we created a convolved image $f$ using a $5\times 5$ filter $h$, and a noise-corrupted missing-data {($60\%$ of pixels retained)} observation $\y$. We implemented GPDC to recover $x$ from $\y$ both when $h$ is \textbf{known} and when it is \textbf{unknown}, \red{therefore, for each image, we considered i) the case where the discrete filter $h$ is known and  used by GPDC, and ii) the case where $h$ is learnt from data via maximum likelihood. We present results using three different filters: constant, unimodal and diagonal for the main body of the manuscript (Fig.~\ref{fig:CIFAR}), and other additional shapes in the Appendix. In all cases, we assumed an SE kernel for the source (ground truth image) and for each case we learnt the lenghtscale $l$, the magnitude $\sigma$ and the noise parameter $\sigma_n$. For the blind cases the discrete filter $h$ was also learnt.}

GPDC was compared against the Wiener filter (from \texttt{scikit-image}), applied to the \textbf{complete}, \textbf{noiseless}, image  $f$. The reason to compare the proposed GPDC to the standard Wiener filter  applied over the true image $f$ instead of the generated observations $\y$ is twofold: i)  it serves as a benchmark that uses all the image information, and ii) the Wiener is unable to deal with missing data (as in $\y$) in its standard implementation.
Fig.~\ref{fig:CIFAR} shows the results for combinations of 3 different images from CIFAR-10 \cite{cifar10} and up to three different filters for each image. Notice how both non-blind and blind GPCD were able to provide reconstructions that are much more representative of the true image ($x$) than those provided by the Wiener filter. The superior performance of GPDC, against the Wiener benchmark, is measured through the mean squared error with respect to the true image $x$ as a function of the amount of seen data (rightmost plot in Fig.~\ref{fig:CIFAR}). The Appendix includes additional examples with different images and filters.

\iffalse %old figure
\begin{figure*}
\adjustbox{valign=c}{%
\begin{minipage}[t]{.30\linewidth}
\caption{Deconvolution of 3 images from lower-resolution, noisy and incomplete observations. Left to right: true  image  $x$, convolution $f$, observation $\y$ ($25\%$ missing pixels in red), GPDC estimate, blind-GPDC estimate, Wiener estimate and performance of GPDC versus amount of observed pixels (error bars for 20 trials).}
\label{fig:CIFAR}
\end{minipage}}%
\adjustbox{valign=c}{%
\begin{minipage}[t]{.7\linewidth}
\vspace{-1.5em}
  \begin{subfigure}{\linewidth}
  \centering
\includegraphics[width=0.95\textwidth]{img/emu.pdf}
  %\caption{test subfigure one}
  \end{subfigure}%\par\bigskip

  \begin{subfigure}{\linewidth}
  \centering
    \includegraphics[width=0.95\textwidth]{img/chequered.pdf}
  %\caption{test subfigure two}
  \end{subfigure}%\par\bigskip

  \begin{subfigure}{\linewidth}
  \centering
    \includegraphics[width=0.95\textwidth]{img/woven.pdf}     
  %\caption{test subfigure three}
  \end{subfigure}
\end{minipage}}
\end{figure*}

\fi %old figure

\iffalse %old figure
\begin{figure*}
  \centering
\includegraphics[width=0.75\textwidth]{img/emu.pdf}
    \includegraphics[width=0.8\textwidth]{img/chequered.pdf}
    \includegraphics[width=0.8\textwidth]{img/woven.pdf}     
\caption{Deconvolution of 3 images from lower-resolution, noisy and incomplete observations. Left to right: true  image  $x$, convolution $f$, observation $\y$ ($25\%$ missing pixels in red), GPDC estimate, blind-GPDC estimate, Wiener estimate and performance of GPDC versus amount of observed pixels (error bars for 20 trials).}
\label{fig:CIFAR}
\end{figure*}
\fi %old figure

\

\section{Discussion and Conclusions} 
\label{sec:discussion}

We have proposed Gaussian process deconvolution (GPDC), a methodology for Bayesian nonparametric deconvolution of continuous-time signals which builds on a GP prior over the latent source $x$ and thus allows us to place error bars on the deconvolution estimate. We studied the direct generative model (the conditions under which it is well defined and how to sample from it), the inverse problem (we provided necessary conditions for the successful recovery of the source), the blind deconvoltion case, and we also showed illustrative examples and experimental validation on dereverberation and super-resolution settings. A key point of our method in connection with the classical theory can be identified by analysing eq.~\eqref{eq:deconv_mean} in the Fourier domain, where we can interpret GPDC as a nonparametric, missing-data-able, extension of the Wiener deconvolution. In this sense, the proposed method offers a GP interpretation to classic deconvolution, thus complementing  the deconvolution toolbox with the abundant resources for GPs such as off-the-shelf covariance functions, sparse approximations, and a large body of dedicated software. 

\red{When constrained to the discrete and non-missing data scenario, our method reduces to the standard Wiener deconvolution, known in the literature \cite{wiener1964extra}; thus, we have focused on comparing GPDC to the Wiener deconvolution method. However, despite the interpretation of GPDC from the viewpoint of classic deconvolution, we have considered a setting that is entirely different to that of digital communications. We have focused on estimating a continuous-time object (the source $x$) using a set of finite observations $\y$, which are to be understood as noise-corrupted and missing-part realisations of the convolution process $f = x \star h$, and to quantify the uncertainty of this estimate. Since we have limited information about the source ($x$ can only be known through $\y$ as per eq.~\eqref{eq:conv_y}), we take a Bayesian approach and impose a prior over $x$, this prior is a Gaussian process. To the best of our knowledge, there is no prior work addressing this setting in conceptual terms from a theoretical perspective, with the exception of a few works that have applied it to specific problems (see \cite{ramos2015bayesian,tobar17a,arjas2020blind,paciorek2006spatial,amt-15-3843-2022}).

As far as novel deconvolution methods are concerned, the trend in the literature is to merge neural networks with classical concepts such prior design \cite{ren2020neural} and the Wiener filter \cite{NEURIPS2020_0b8aff04}, and are mainly applied to images, perhaps following the strong and steady advances in deep learning for computer vision. However, despite some recent applications of GPs to the deconvolution problem in particular settings, there are, to the best of our knowledge, no method addressing the recovery of a continuous-time (or continuous-space in the case of images) object from a finite number of noisy observations for both the blind and non-blind cases with the two main properties of our proposal: i) taking a Bayesian standpoint that allows for the determination of error bars, and ii) providing the Fourier-inspired guarantees in Sec.~\ref{sec:deconvolution}\ref{sec:deconv_possible}. In this sense, we claim that ours is the first application-agnostic study of the method which established the conditions for the proper definition of the hierarchical model, the capacity of the deconvolution procedure, and a principled quantification of uncertainty.}

We hope that our findings pave the way for applications in different sciences and also motivates the GP community to consider extensions of our study. In particular, we envision further research efforts to be dedicated in the following directions: 
\begin{itemize}
	\item to develop sufficient (rather than only necessary) conditions for deconvolution to complement Theorem \ref{teo:recover} and thus connecting with the interface between spectral estimation and GPs \cite{tobar18}. %\cite{tobar18,tobar21b}. 
	%\item to incorporate GP-convolutions for the problem of \emph{dynamic time warping} \cite{berndt1994using} from the perspective of distances among time series, thus connecting this challenging scenario with Bayesian inference, 
	\item to construct sparse GPs by controlling temporal correlation through trainable, multi-resolution, convolutions. This would provide computationally efficient sparse GPs with clear reconstructions guarantees.
	\item to develop a sparse GP version of the presented methodology in order to apply GPDC for deconvolution of large datasets
\end{itemize}

\enlargethispage{20pt}

%\ethics{Insert ethics text here.}

\dataccess{All data used in this article is either synthetic or public. Experiment 1 uses audio data from \url{http://www.mcsquared.com/reverb.htm}, while Experiment 2 uses CIFAR-10 \cite{cifar10}.}

%\aucontribute{F.T.: conceptualization,  investigation, methodology, design and implementation of experiments,  writing    and editing, supervision;    J.F.S.:  investigation, writing    and editing;   A.R.: investigation, methodology, design and implementation of experiments,  writing and editing; }

%\competing{Insert competing text here.}

\funding{We acknowledge financial support from Google and the following ANID-Chile grants: Fondecyt-Regular 1210606, the Advanced Center for Electrical and Electronic Engineering (Basal FB0008) and the Center for Mathematical Modeling (Basal FB210005).}

%\ack{Insert acknowledgment text here.}

%\disclaimer{Insert disclaimer text here.}

%%%%%%%%%% Insert bibliography here %%%%%%%%%%%%%%
\bibliographystyle{RS}
  \bibliography{references.bib}
%\printbibliography
\section*{Appendix}

This section contains the detailed calculation needed for Theorem 1, as well as detailed figures to the audio deconvolution experiment and more examples for the image deconvolution experiment. 
\subsection{Extended proof of Lemma 1}

\begin{lemma} 
  \label{lemma:K_f}
  If the convolution filter $h$ and the %stationary 
  covariance $K_x$ are both integrable, then $K_f(t)$ is integrable.
\end{lemma}

\begin{proof}
  This follows in the same vein as the standard proof of integrability of the convolution between two functions with a slight modification, since the definition of $K_f(t)$---eq.~(5) in the article---comprises the composition of two convolutions rather than just one. Therefore, using Fubini Thm and the triangle inequality (twice), we have 
  \begin{align*}
  \int_\R|K_f(t)|\dt& = \int_\R\left| \int_{\R} h(\tau')\int_{\R}h(\tau) K_x(\tau - (\tau'-(t)) \dtau\dtau'\right|\dt &&\text{[Fubini on eq.~(5))]}\\ 
      & \leq \int_\R \int_{\R} |h(\tau')|  \int_{\R} \left|h(\tau)\right| \left|K_x(\tau - (\tau'-(t)) \right|\dtau\dtau'\dt&&\text{[triangle ineq. twice]}\\
      & =  \int_{\R} |h(\tau')|  \int_{\R} \left|h(\tau)\right|  \int_\R  \left|K_x(\tau - (\tau'-(t)) \right|\dt \dtau\dtau' &&\text{[Fubini]}\\
      & =  ||h||_1||h||_1 ||K_x||_1 && \\ 
      & < \infty.&& \text{[$h,K_x\in L_1$]}
  \end{align*}
\end{proof}

\subsection{Calculations for Theorem 1}

\begin{align}
	  	\E{\hat x_w (\xi)|\y} &= \E{\fourier{x(t)w(t)}|\y}\\
	  	 &= \fourier{\E{x(t)|\y}w(t)} \nonumber \\
	  	&= \fourier{\left(m_x(t) + K_{xy}(t,\t)K_y^{-1}(\t)\y\right) w(t)} \nonumber \\
	  	&= \left(m_x(t) + \fourier{K_{xy}(t,\t)}K_y^{-1}(\t)\y\right) \star \hat w(\xi) \nonumber \\
	  	&= \left(\hat m_x(\xi) + \hat{K}_{xy}(\xi)e^{-j2\pi\t\xi}K_y^{-1}(\t)\y\right)\star \hat w(\xi) \nonumber \\
	  	&= \left(\hat m_x(\xi) + \hat{K}_{x}(\xi)\hat{h}(\xi)e^{-j2\pi\t\xi}K_y^{-1}(\t)\y\right)\star \hat w(\xi)\nonumber
	  \end{align}

	  	\begin{align}
	  	\V{\hat x_w (\xi)|\y} &=\V{\fourier{x(t)w(t)}|\y}\\
	  	 &=\E{\overline{\fourier{x(t)w(t) - \E{x(t)w(t)|\y}}} \fourier{x(t)w(t)  - \E{x(t)w(t)|\y} }| \y}\nonumber\\
	    \text{[lin. exp.]}	&= \E{\overline{\fourier{(x(t) - \E{x(t)|\y})w(t)}} \fourier{(x(t) - \E{x(t)|\y})w(t) }| \y}\nonumber\\
	  \text{[conv. thm]}	&= \E{\fourier{(x(-t) - \E{x(-t)|\y})w(-t) \star (x(t) - \E{x(t)|\y})w(t)}| \y}\nonumber\\
	  \text{[def. conv.]}	&= \E{\fourier{\int(x(-\tau) - \E{x(-\tau)|\y})w(-\tau)  (x(t-\tau) - \E{x(t-\tau)|\y})w(t-\tau)}| \y}\nonumber\\
	  \text{[lin. conv.]}	&= \fourier{\int w(-\tau) \E{(x(-\tau) - \E{x(-\tau)|\y}) (x(t-\tau) - \E{x(t-\tau)|\y})| \y} w(t-\tau)}\nonumber\\
	  \text{[def. cov.]}	&= \fourier{\int w(-\tau) \V{x(-\tau),x(t-\tau)|\y} w(t-\tau)}\nonumber\\
	  	\text{[def. cov]}	&= \fourier{ \int \left(K_{x} (t) - K_{xy} (-\tau, \t) K_{y}^{-1}K_{yx} (\t, t-\tau) \right)w(-\tau)  w(t-\tau) \dtau}\nonumber\\
	  		 &= \hat{K}_x(\xi)\star|\hat{w}(\xi)|^2 - \fourier{ \int w(-\tau)K_{xy} (-\tau, \t) K_y^{-1} K_{yx} (\t, t-\tau)  w(t-\tau) \dtau}\nonumber\\
	  		  &= \hat{K}_x(\xi)\star|\hat{w}(\xi)|^2 -  \overline{ \fourier{ w(t)K_{xy} (t, \t) } } K_y^{-1} \fourier{ K_{yx} (\t, t)  w(t) }\nonumber\\
	  		  &= \hat{K}_x(\xi)\star|\hat{w}(\xi)|^2 -  \overline { \hat w(\xi)   \star  \hat K_x(\xi)\hat h(\xi)e^{-j2\pi\xi\t}  }   K_y^{-1} \hat w(\xi)   \star  \hat K_x(\xi)\hat h(\xi)e^{-j2\pi\xi\t} \nonumber\\
	  		  &= \hat{K}_x(\xi)\star|\hat{w}(\xi)|^2  -  \| \hat w(\xi)   \star  \hat K_x(\xi)\hat h(\xi)e^{-j2\pi\xi\t}  \|_{ K_y^{-1}} \nonumber
	  \end{align}  

\subsection{Extended audio experiment}
For the audio de-reverberation experiment (see Fig.~\ref{fig:reverb} in the main body of the article), we also show the estimates of all methods considered in detail. First, Fig.~\ref{fig:audio_true} shows the true signals (and filter), while Fig.~\ref{fig:audio_GPDC} shows the estimate of the proposed GPDC (mean and 95\% error bars). Figs.~\ref{fig:audio_Wiener} and \ref{fig:audio_invFT} show the estimates provided by the Wiener and Inverse FT methods respectively, considered as benchmarks to GPDC in this experiment. We also clarify that, as these last two methods perform direct deconvolution in the frequency domain, they do not maintain the phase of the signal and thus they need to be aligned with the true signal---we have aligned them for visualisation purposes. 

Based on the estimates of the source $x$ provided by the Inverse FT, Wiener and proposed GPDC (Figs.~\ref{fig:audio_invFT}, \ref{fig:audio_Wiener} and \ref{fig:audio_GPDC} respectively), we can see how the worst reconstruction provided by the Inverse FT (which does not consider the fact that the source is stochastic). Then, the reconstruction provided by Wiener improves over the Inverse FT, since Wiener considers the dynamics of both the filter and the signal, however, it lacks a \emph{likelihood model}, meaning that it is not able to discriminate between the (noisy) observations $\y$ and the convolution $f$. Lastly, observe that the proposed GPDC performed better that the benchmarks due the fact that  it incorporates both the dynamic behaviour of the source (through $K_x$) and the noise in the observations $\y$.

\begin{figure}[t]
 \centering
\includegraphics[width=0.9\textwidth]{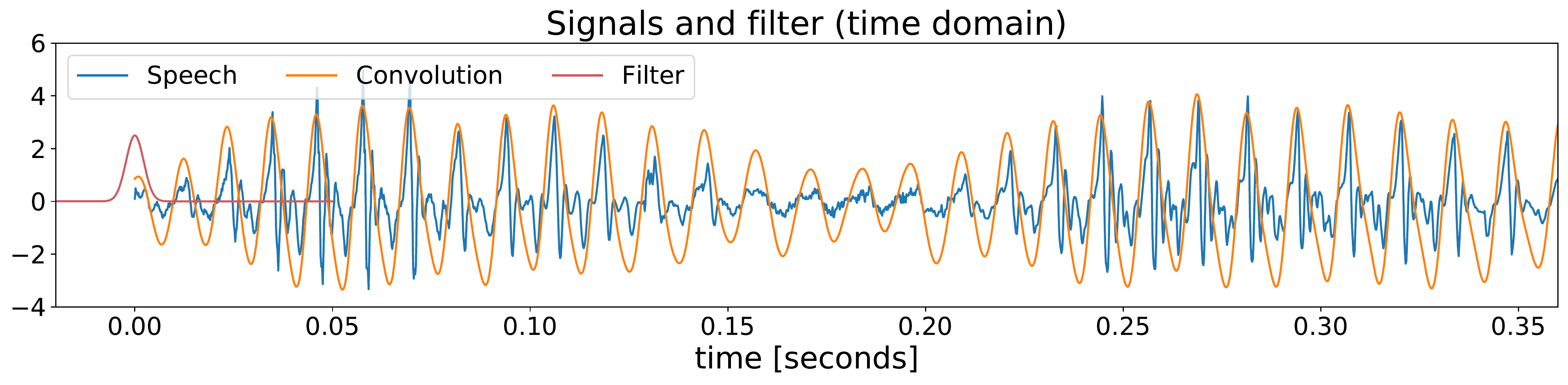}
\caption{True audio signals (source and convolution)  together with the convolution filter.}
\label{fig:audio_true}
\end{figure}
\begin{figure}[t]
\includegraphics[width=0.9\textwidth]{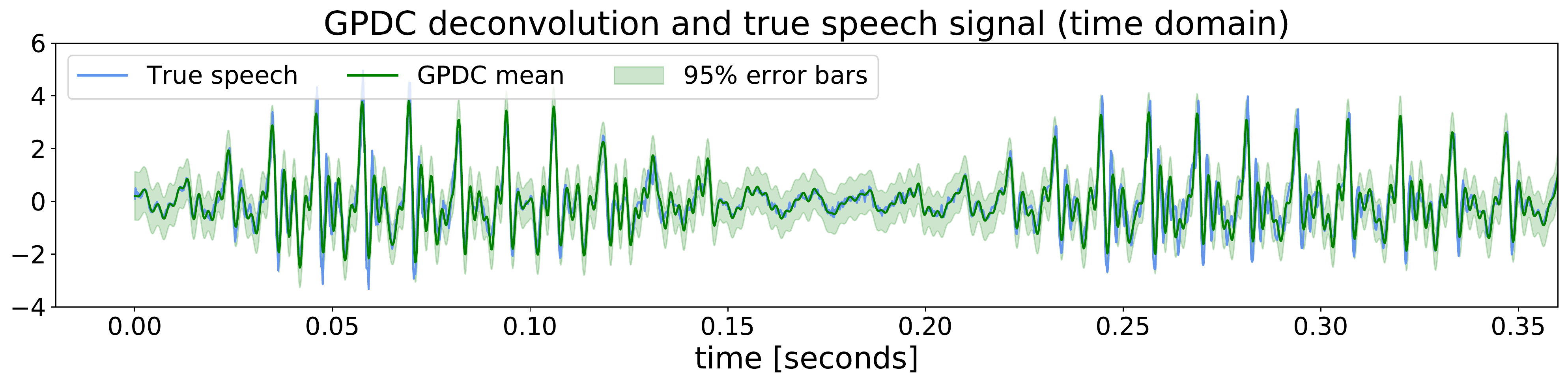}
\caption{Deconvolution of a speech signal: GPDC estimate alongside true source and  95\% error bars.}
\label{fig:audio_GPDC}
\end{figure}
\begin{figure}[t]
\includegraphics[width=0.9\textwidth]{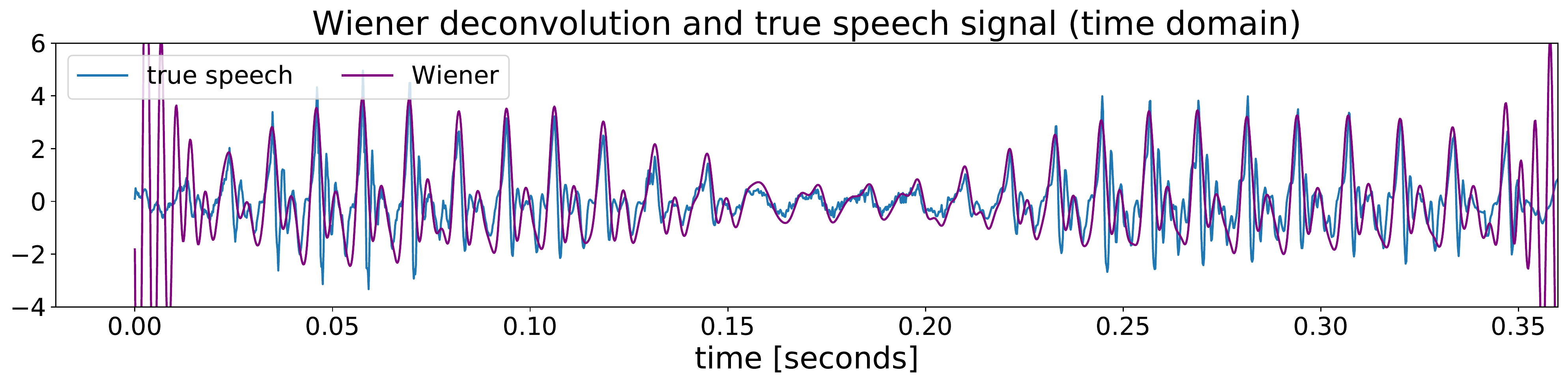}
\caption{Deconvolution of a speech signal: Wiener estimate alongside true source.}
\label{fig:audio_Wiener}
\end{figure}
\begin{figure}[t]
\includegraphics[width=0.9\textwidth]{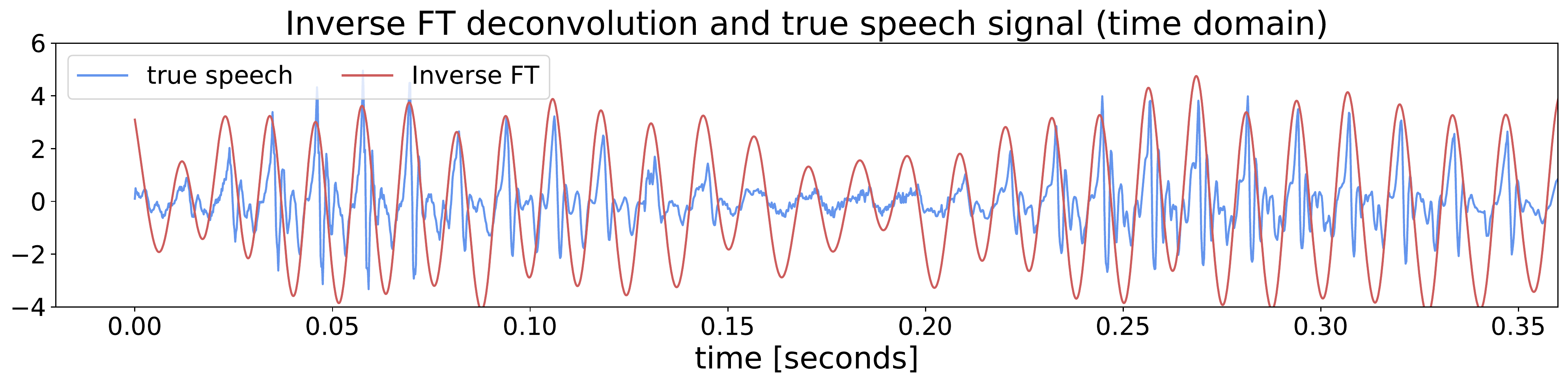}
\caption{Deconvolution of a speech signal: Inverse FT estimate alongside true source.}
\label{fig:audio_invFT}
\end{figure}

\subsection{Extended image experiment}

In this section, we include further simulations for the image deconvolution or \emph{superresolution} experiments on additional test images and filters. This experiment aimed to validate the proposed GPDC and the blind GPDC (where the filter $h$ is unknown and thus learnt from the images) against the Wiener deconvolution method. We considered the following $(32\times32)$ images from the CIFAR-10 Dataset: an airplane (Fig.~\ref{fig:CIFAR-PLANE}), an emu (Fig.~\ref{fig:CIFAR-EMU}), a bird (Fig.~\ref{fig:CIFAR-BIRD}), a horse (Fig.~\ref{fig:CIFAR-RIDER}) and a second airplane (Fig.~\ref{fig:CIFAR-SWISS}). For all experiments, we considered 5 different filters $h$ of size $(5 \times 5)$ pixels, these filters are shown in each figure from top to bottom colour-coded in grey scale  between 0 and 1, and  are given by:
\begin{enumerate}
  \item $h_0$: a constant filter,
  \item $h_1$: a filter with a large (i.e., 1) at the origin and close to zero around the origin,
  \item $h_2$: a radial filter that decays like a Gaussian RBF,
  \item $h_3$: a filter composed of random values  uniformly distributed between 0 and 1, and
  \item $h_4$: a diagonal filter.
\end{enumerate}
Lastly, for all images and filters, GPDC and blind-GPDC where implemented with a number of observations  ranging from 10\% to 90\% (for the performance plot), where the Figures exhibit the case for  30\% missing data.

To understand the performance of the proposed GPDC and its blind variant against Wiener, recall that \textbf{the Wiener benchmark was implemented on the complete and noiseless image}, whereas  the proposed model had access only to a missing-data and noisy version of the image. Let us observe that for the filters $h_1$ and $h_2$, although the Wiener outperforms GPDC and blind-GPDC, the proposed models present a monotonic improvement wrt to the number of observations and even the blind GPDC is visually able to recover the true image. For more complex filters ($h_0,h_2,h_4$) we can see how both GPDC  methods greatly outperform the  Wiener benchmark whenever more than a 25\% of the image is available. This behaviour is  consistently found for all images, thus  validating the robustness of the proposed GPDC methods while dealing with  different images.

In all the following figures, the results are shown in the following order from left to right: filter  (values between 0 and 1 colour-coded in grey scale), true  image  $x$, convolution $f$, observation $\y$ ($30\%$ missing pixels in red), GPDC estimate, blind-GPDC estimate, Wiener estimate and performance of GPDC versus amount of observed pixels (error bars for 20 trials). For each figure, each row is a realisation of the same experiment but with a different filter.

\begin{figure}[t]
  \centering
  \begin{subfigure}{\linewidth}
  \centering
\includegraphics[width=1\textwidth]{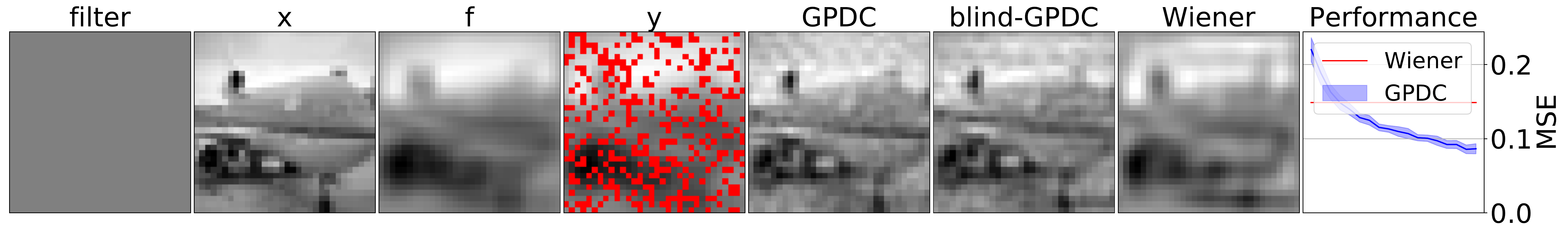}
  %\caption{test subfigure one}
  \end{subfigure}%\par\bigskip

  \begin{subfigure}{\linewidth}
  \centering
    \includegraphics[width=1\textwidth]{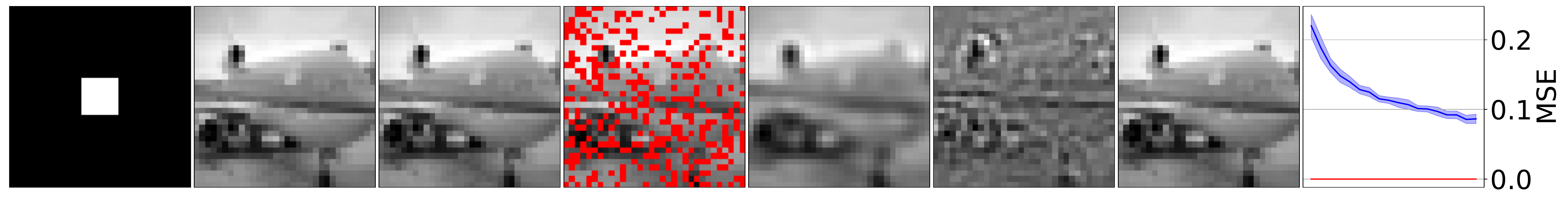}
  %\caption{test subfigure two}
  \end{subfigure}%\par\bigskip

  \begin{subfigure}{\linewidth}
  \centering
    \includegraphics[width=1\textwidth]{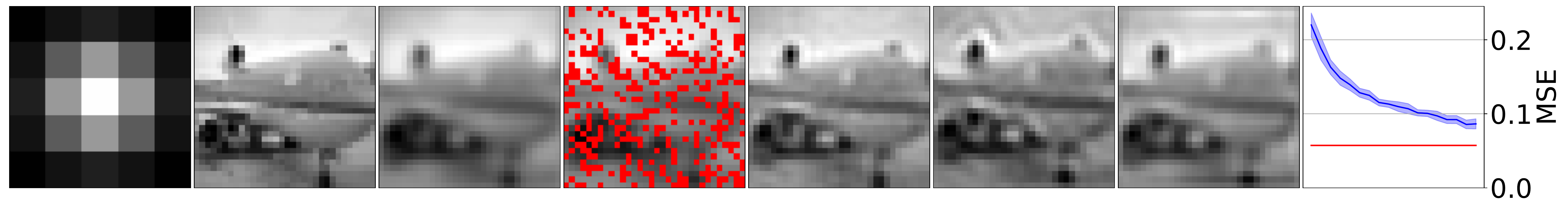}     
  %\caption{test subfigure three}
  \end{subfigure}

  \begin{subfigure}{\linewidth}
  \centering
    \includegraphics[width=1\textwidth]{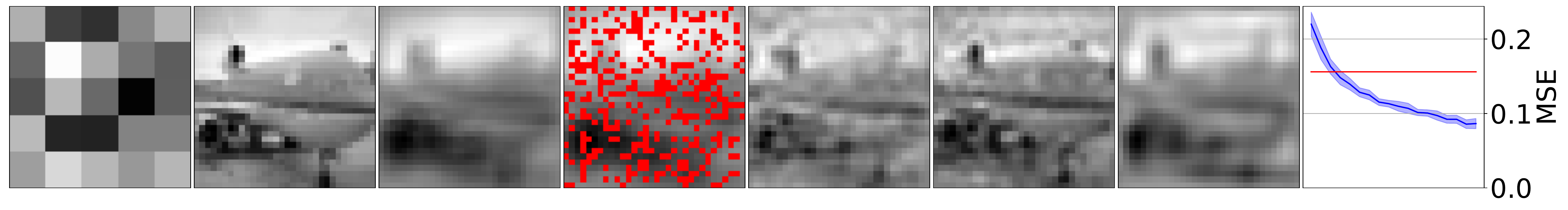}     
  %\caption{test subfigure three}
  \end{subfigure}

  \begin{subfigure}{\linewidth}
  \centering
    \includegraphics[width=1\textwidth]{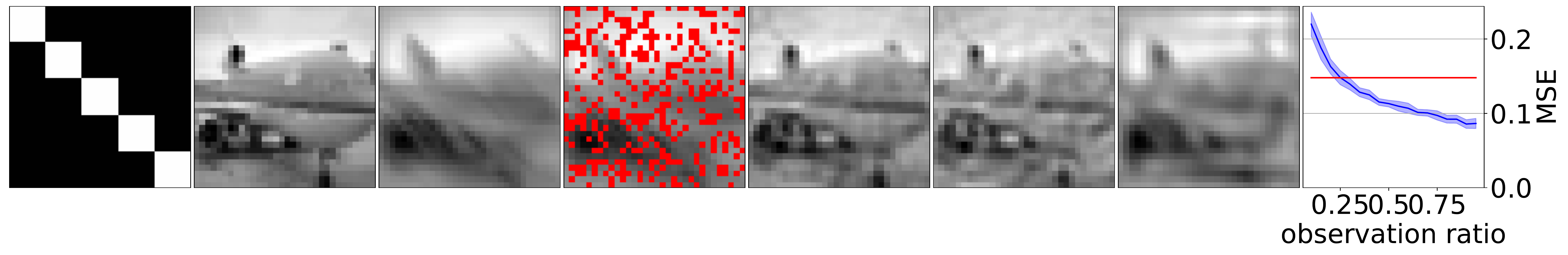}     
  %\caption{test subfigure three}
  \end{subfigure}
 \caption{Picture of a plane.}
 \label{fig:CIFAR-PLANE}

\end{figure}

\begin{figure}[t]
  \begin{subfigure}{\linewidth}
  \centering
\includegraphics[width=1\textwidth]{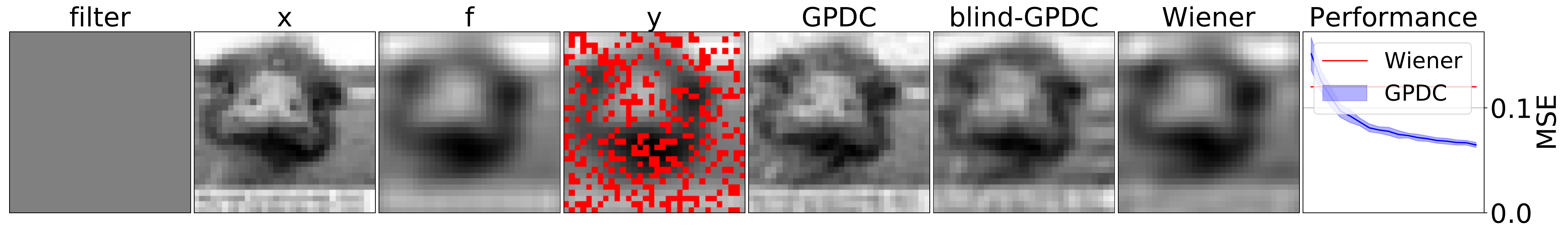}
  %\caption{test subfigure one}
  \end{subfigure}%\par\bigskip

  \begin{subfigure}{\linewidth}
  \centering
    \includegraphics[width=1\textwidth]{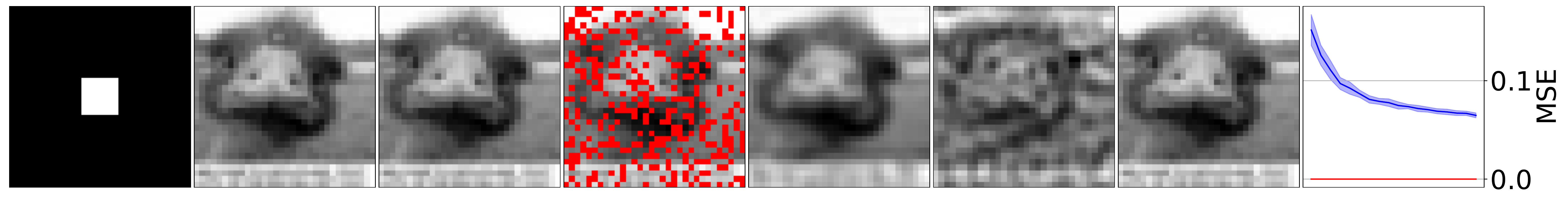}
  %\caption{test subfigure two}
  \end{subfigure}%\par\bigskip

  \begin{subfigure}{\linewidth}
  \centering
    \includegraphics[width=1\textwidth]{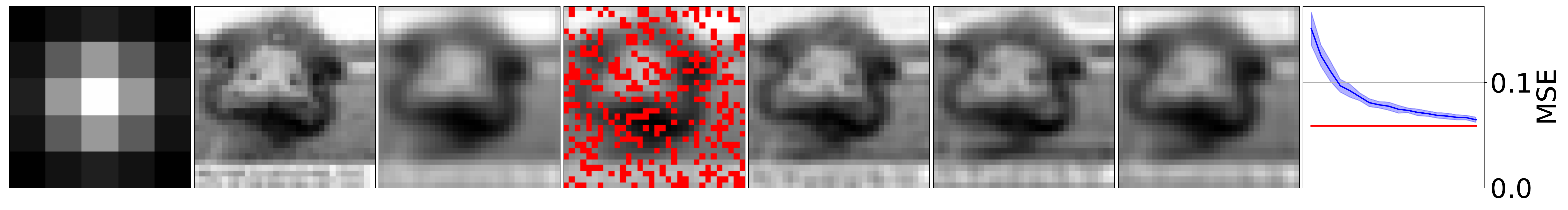}     
  %\caption{test subfigure three}
  \end{subfigure}

  \begin{subfigure}{\linewidth}
  \centering
    \includegraphics[width=1\textwidth]{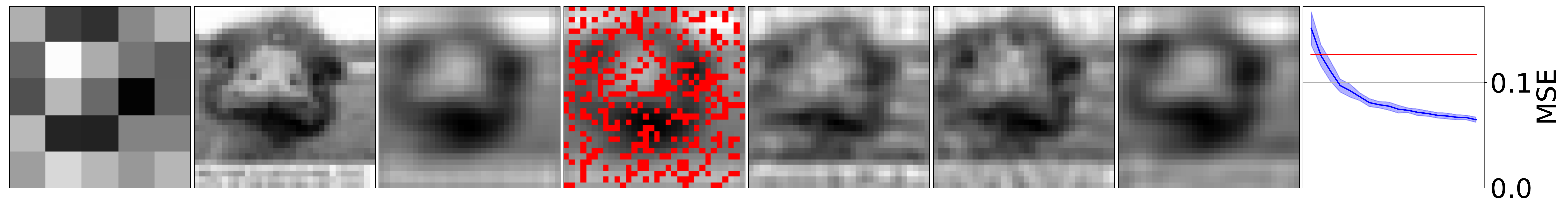}     
  %\caption{test subfigure three}
  \end{subfigure}

  \begin{subfigure}{\linewidth}
  \centering
    \includegraphics[width=1\textwidth]{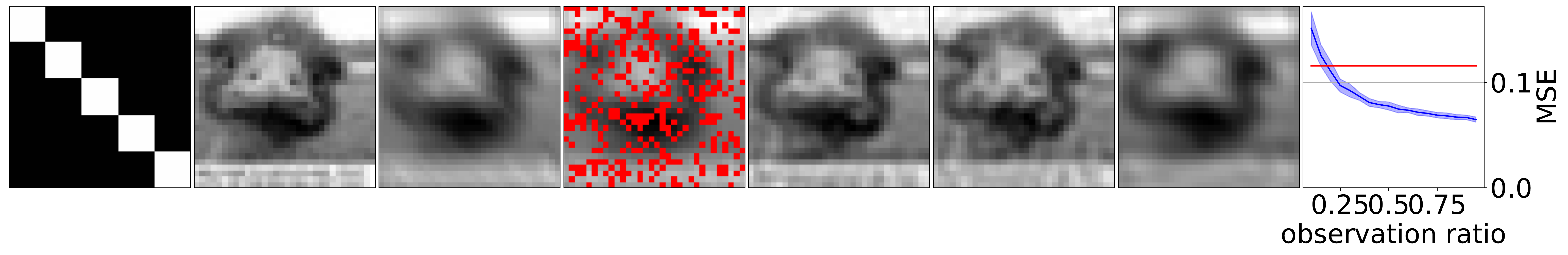}     
  %\caption{test subfigure three}
  \end{subfigure}
 \caption{Picture of an Emu.}
\label{fig:CIFAR-EMU}
\end{figure}

\begin{figure}[t]
  \begin{subfigure}{\linewidth}
  \centering
\includegraphics[width=1\textwidth]{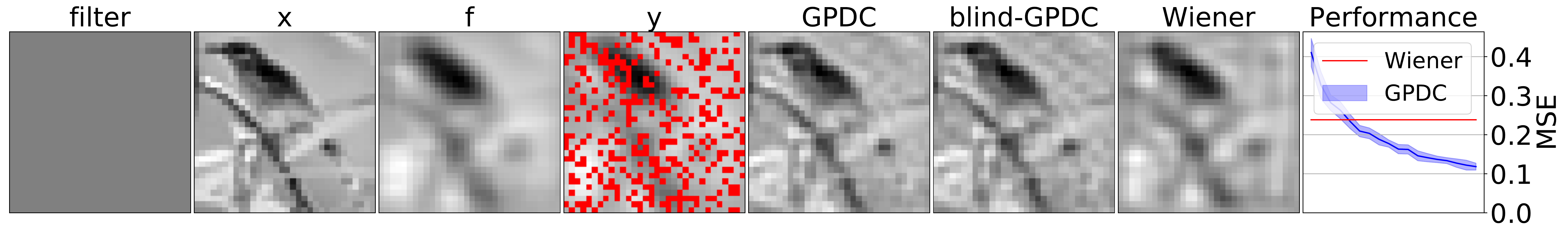}
  %\caption{test subfigure one}
  \end{subfigure}%\par\bigskip

  \begin{subfigure}{\linewidth}
  \centering
    \includegraphics[width=1\textwidth]{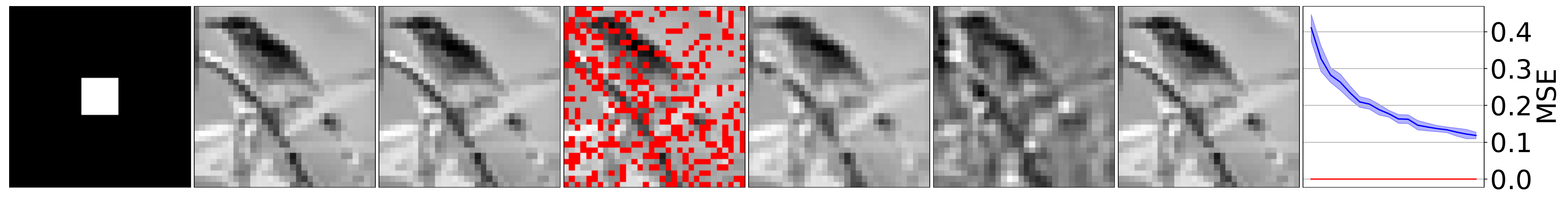}
  %\caption{test subfigure two}
  \end{subfigure}%\par\bigskip

  \begin{subfigure}{\linewidth}
  \centering
    \includegraphics[width=1\textwidth]{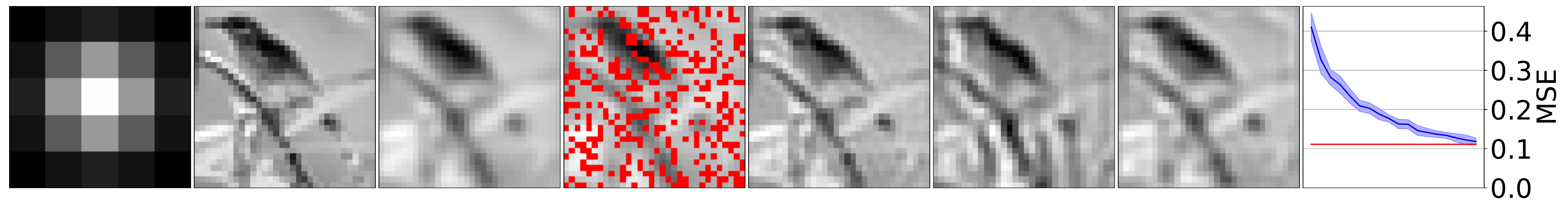}     
  %\caption{test subfigure three}
  \end{subfigure}

  \begin{subfigure}{\linewidth}
  \centering
    \includegraphics[width=1\textwidth]{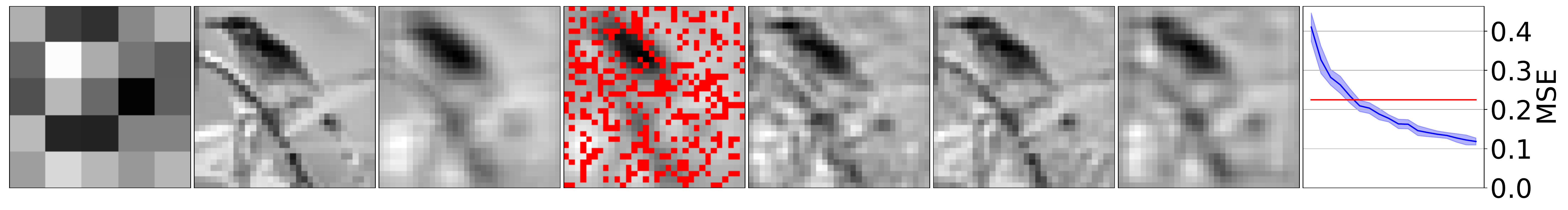}     
  %\caption{test subfigure three}
  \end{subfigure}

  \begin{subfigure}{\linewidth}
  \centering
    \includegraphics[width=1\textwidth]{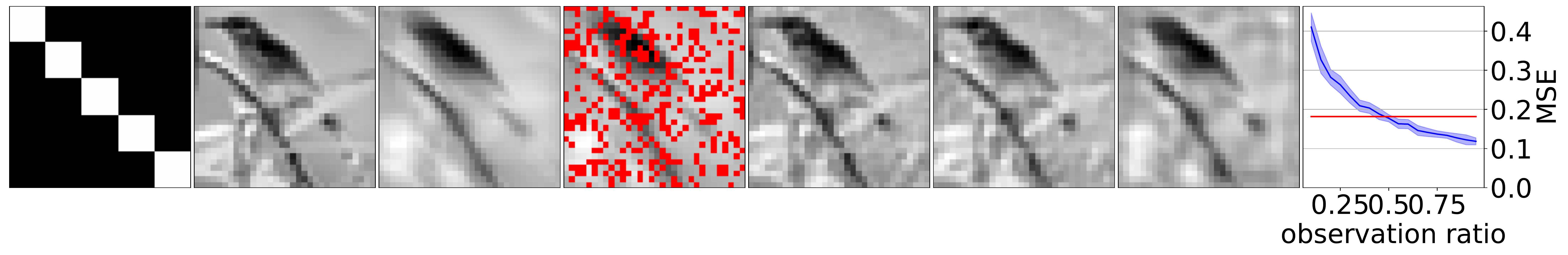}     
  %\caption{test subfigure three}
  \end{subfigure}
 \caption{Picture of a bird.}
\label{fig:CIFAR-BIRD}

\end{figure}

\begin{figure}[t]
  \begin{subfigure}{\linewidth}
  \centering
\includegraphics[width=1\textwidth]{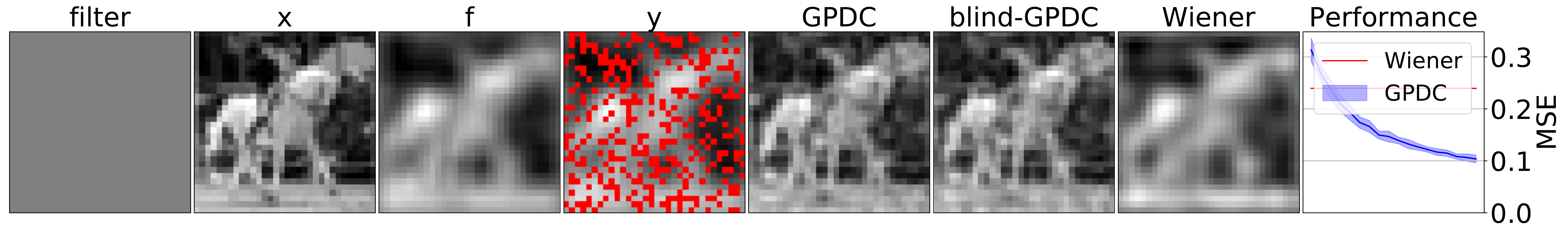}
  %\caption{test subfigure one}
  \end{subfigure}%\par\bigskip

  \begin{subfigure}{\linewidth}
  \centering
    \includegraphics[width=1\textwidth]{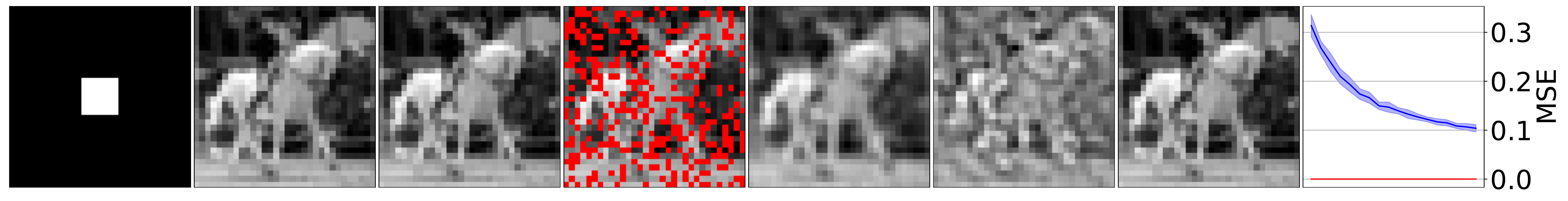}
  %\caption{test subfigure two}
  \end{subfigure}%\par\bigskip

  \begin{subfigure}{\linewidth}
  \centering
    \includegraphics[width=1\textwidth]{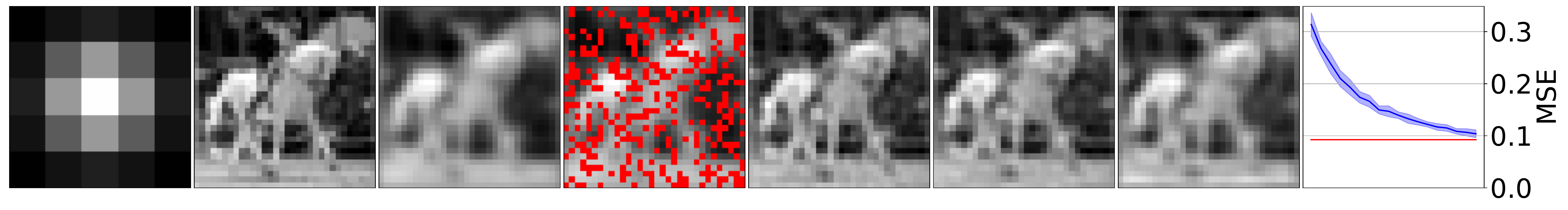}     
  %\caption{test subfigure three}
  \end{subfigure}

  \begin{subfigure}{\linewidth}
  \centering
    \includegraphics[width=1\textwidth]{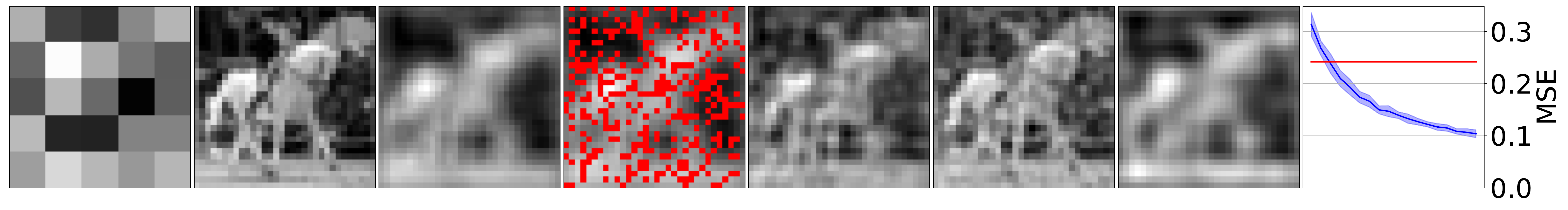}     
  %\caption{test subfigure three}
  \end{subfigure}

  \begin{subfigure}{\linewidth}
  \centering
    \includegraphics[width=1\textwidth]{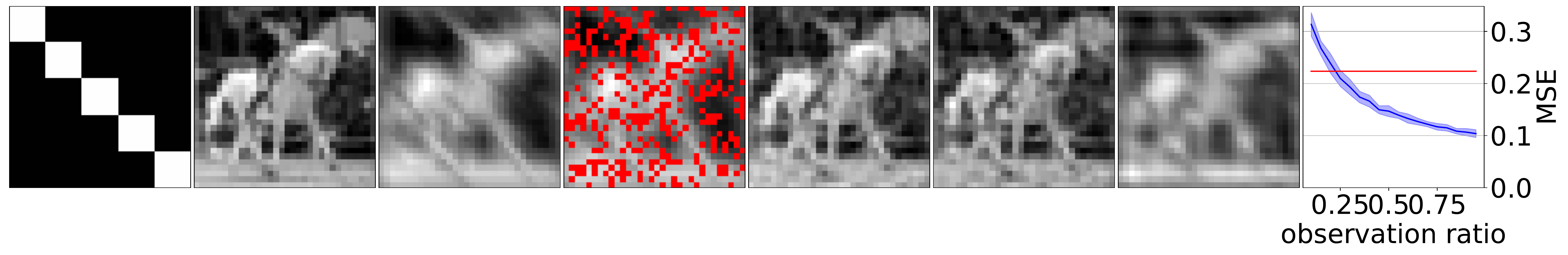}     
  %\caption{test subfigure three}
  \end{subfigure}
 \caption{Picture of a horse.}
\label{fig:CIFAR-RIDER}

\end{figure}

\begin{figure}[t]
  \begin{subfigure}{\linewidth}
  \centering
\includegraphics[width=1\textwidth]{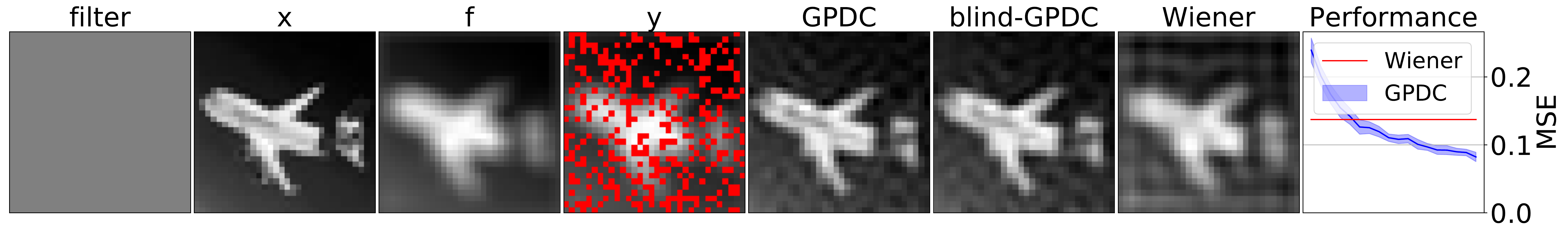}
  %\caption{test subfigure one}
  \end{subfigure}%\par\bigskip

  \begin{subfigure}{\linewidth}
  \centering
    \includegraphics[width=1\textwidth]{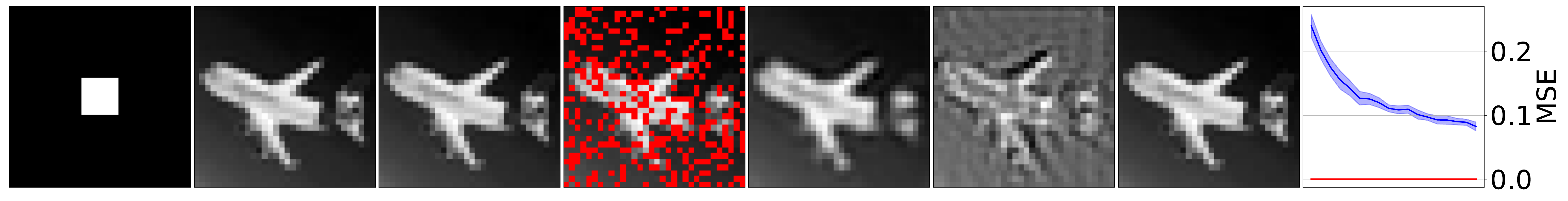}
  %\caption{test subfigure two}
  \end{subfigure}%\par\bigskip

  \begin{subfigure}{\linewidth}
  \centering
    \includegraphics[width=1\textwidth]{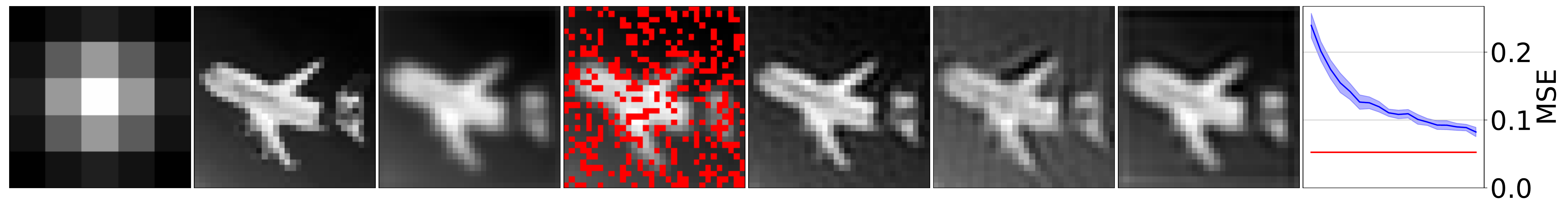}     
  %\caption{test subfigure three}
  \end{subfigure}

  \begin{subfigure}{\linewidth}
  \centering
    \includegraphics[width=1\textwidth]{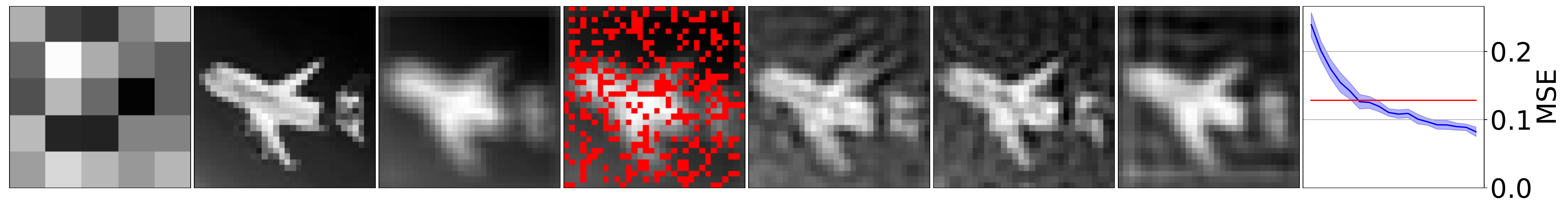}     
  %\caption{test subfigure three}
  \end{subfigure}

  \begin{subfigure}{\linewidth}
  \centering
    \includegraphics[width=1\textwidth]{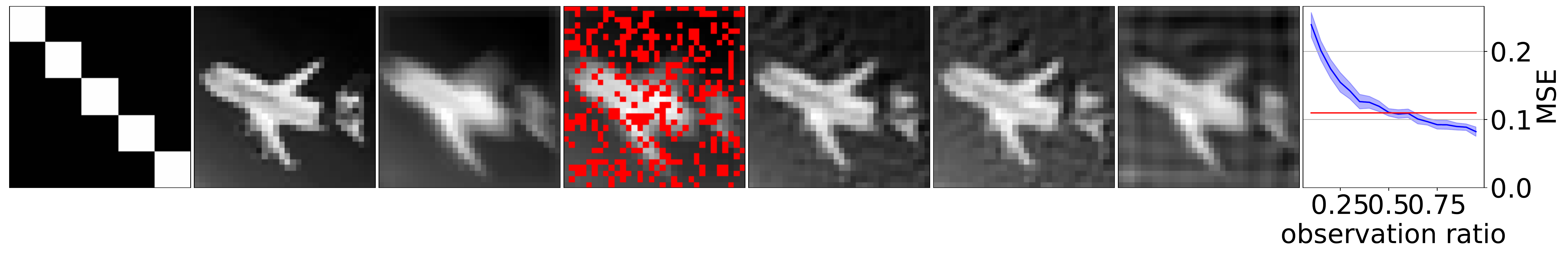}     
  %\caption{test subfigure three}
  \end{subfigure}
 \caption{Second picture of a plane.}
\label{fig:CIFAR-SWISS}

\end{figure}

\end{document}